\documentclass[11pt]{article}

\usepackage{amsmath,amsfonts,amssymb}
\usepackage{mathtools}
\usepackage{amsthm} 
\usepackage{latexsym}
\usepackage{relsize}
\usepackage{titlesec}

\usepackage{multirow,bm}
\usepackage{makecell}
\usepackage{longtable}
\usepackage{threeparttable}
\usepackage{longtable}
\usepackage{booktabs}
\usepackage[referable]{threeparttablex}
\usepackage{footnote}
\usepackage{dirtytalk}

\usepackage[ruled,vlined]{algorithm2e}

\usepackage{authblk}
\usepackage{enumitem}
\usepackage{textgreek}
\usepackage{units}
\usepackage{amsmath}
\usepackage{amsthm}
\usepackage{amssymb}
\usepackage{subfig}
\usepackage{color}
\usepackage[english]{babel}
\usepackage{graphicx}
\usepackage{grffile}
\usepackage{wrapfig,epsfig}
\usepackage{epstopdf}
\usepackage{url}
\usepackage{color}
\usepackage{epstopdf}
\usepackage[T1]{fontenc}
\usepackage{bbm}
\usepackage{comment}
\usepackage{dsfont}

\usepackage{tikz}
\usepackage{hyperref}  
\hypersetup{colorlinks=true,citecolor=blue,linkcolor=blue} 
\usetikzlibrary{arrows}

\usepackage[capitalise]{cleveref}
\usepackage{xcolor}
\usepackage{dsfont}

\newcommand{\ignore}[1]{}

\theoremstyle{plain}
\newtheorem{theorem}{Theorem}
\newtheorem{lemma}[theorem]{Lemma}
\newtheorem{corollary}[theorem]{Corollary}

\newtheorem{claim}[theorem]{Claim}
\newtheorem{fact}[theorem]{Fact}

\newtheorem{assumption}{Assumption}

\newtheorem*{theorem*}{Theorem}
\newtheorem*{lemma*}{Lemma}
\newtheorem*{corollary*}{Corollary}
\newtheorem*{proposition*}{Proposition}
\newtheorem*{claim*}{Claim}
\newtheorem*{fact*}{Fact}
\newtheorem*{observation*}{Observation}
\newtheorem*{assumption*}{Assumption}

\theoremstyle{definition}
\newtheorem{definition}[theorem]{Definition}

\newtheorem*{definition*}{Definition}
\newtheorem*{remark*}{Remark}
\newtheorem*{example*}{Example}

 \theoremstyle{plain}
\newtheorem*{theoremaux}{\theoremauxref}
\gdef\theoremauxref{1}

\newenvironment{proof-sketch}{%
  \proof}{\endproof}
\DeclareMathAlphabet{\mathbfsf}{\encodingdefault}{\sfdefault}{bx}{n}

\DeclareMathOperator*{\argmin}{arg\,min}

\newcommand{\wt}[1]{\smash{\widetilde{#1}}}
\newcommand{\wh}[1]{\smash{\widehat{#1}}}
\renewcommand{\O}{O}

\newcommand{\tO}{\wt{\O}}

\newcommand{\trace}{\mathrm{tr}}
\newcommand\simiid{\mathrel{\overset{\makebox[0pt]{\mbox{\normalfont\tiny\sffamily i.i.d}}}{\sim}}}

\DeclareMathOperator{\ind}{\mathds{1}} 

\newcommand{\poly}{\mathrm{poly}}

\renewcommand{\leq}{~\le~}
\renewcommand{\geq}{~\ge~}

\let\oldtfrac\tfrac
\renewcommand{\tfrac}[2]{\smash{\oldtfrac{#1}{#2}}}

\let\nablaold\nabla
\renewcommand{\nabla}{\nablaold\mkern-2.5mu}

\newcommand{\mX}{\mathcal{X}}

\newcommand{\mF}{\mathcal{F}}
\newcommand{\mM}{\mathcal{M}}
\newcommand{\mU}{\mathcal{U}}

\newcommand{\mC}{\mathcal{C}}

\newcommand{\whA}{\widehat{A}}
\newcommand{\whB}{\widehat{B}}
\newcommand{\whw}{\widehat{w}}
\newcommand{\ox}{\overline{x}}
\newcommand{\ou}{\overline{u}}
\newcommand{\oz}{\overline{z}}

\usepackage[margin=1in]{geometry}
\begin{document}

\title{Geometric Exploration for Online Control}
\author[1]{Orestis Plevrakis}
\author[1,2]{Elad Hazan}

\affil[1]{Princeton University, Computer Science Department }
\affil[2]{Google AI Princeton\authorcr
  \tt orestisp, ehazan@princeton.edu}

 \maketitle
  \begin{abstract}
We study the control of an \emph{unknown} linear dynamical system under general convex costs. The objective is minimizing regret vs. the class of disturbance-feedback-controllers, which encompasses all stabilizing linear-dynamical-controllers. In this work, we first consider the case of known cost functions, for which we design the first polynomial-time algorithm with $n^3\sqrt{T}$-regret, where $n$ is the dimension of the state plus the dimension of control input. The $\sqrt{T}$-horizon dependence is optimal, and improves upon the previous best known bound of $T^{2/3}$. The main component of our algorithm is a novel geometric exploration strategy: we adaptively construct a sequence of barycentric spanners in the policy space. Second, we consider the case of bandit feedback, for which we give the first polynomial-time algorithm with $poly(n)\sqrt{T}$-regret, building on Stochastic Bandit Convex Optimization. 

  \end{abstract}

\section{Introduction}

We study the online control of an unknown linear dynamical system under general convex costs. This fundamental problem lies in the intersection of control theory and online learning. It also embodies a central challenge of reinforcement learning: balancing exploration and exploitation in continuous spaces. For this reason, it has recently received considerable attention from the machine learning community. 


Formally, we consider a linear dynamical system (LDS), where the state $x_t\in \mathbb{R}^{d_x}$ evolves as 
\begin{align}
    x_{t+1}=A_*x_t+B_*u_t+w_t,\ \ \ \text{where } x_1=0,
\end{align}
 $u_t\in \mathbb{R}^{d_u}$ is the learner's control input, $w_t\in \mathbb{R}^{d_x}$ is a noise process drawn as $w_t\simiid N(0,I)$, and $A_*,B_*$ are unknown system matrices. The learner applies control $u_t$ at timestep $t$, then observes the state $x_{t+1}$ and suffers cost $c(x_t,u_t)$, where $c$ is a convex function. 
 We consider two forms of cost information for the learner: the case where $c$ is known in advance, and the bandit version where only the scalar cost is observed.
\par Even if the dynamics were known, there are problem instances where the optimal policy is a very complicated function \cite{bemporad2002explicit}. A way to circumvent this is to consider a policy class that is both expressive and tractable, and aim for performing as well as the best policy from that class. The objective that captures this goal is regret, which has recently become the performance metric of choice in online control. The most general policy class, for which there currently exist efficient algorithms with sublinear regret is the class of disturbance-feedback-controllers (DFCs). DFCs encompass the class of stabilizing linear-dynamical-controllers (LDCs)\footnote{For a proof of this see \cite{simchowitz2020improper}.}, which are the "gold standard" in linear control theory as they are known to be optimal for $\mathcal{H}_2$ and $\mathcal{H}_\infty$ control in partially observed LDSs.

\par 
Formally, regret with respect to a class $\Pi$ of  policies is defined as 
\begin{align}\label{eq:regret_intro}
    R_T=\sum_{t=1}^T c(x_t,u_t) -T \min_{\pi\in \Pi} J\left(\pi \right),
\end{align}
where a policy $\pi \in \Pi$ applies control $u_t^\pi $ at state $x_t^\pi$, and letting $\mathbb{E}_\pi$ denote expectation under this policy,
\begin{align}
   J(\pi)=\lim_{T\to \infty} \frac{1}{T}\mathbb{E}_\pi\left[\sum_{t=1}^Tc\left(x_t^\pi,u_t^\pi\right)\right]
\end{align}
is the average infinite-horizon cost of $\pi$. Notice that the regret metric is counterfactual; the cost of the learner is compared to the cost of the best policy in the class, had it been played from the beginning of time!  Henceforth, the most general class of policies we consider are DFCs, which choose their control based upon a window of past disturbances:
\begin{align}\label{eq:DFC}
  u_t=\sum_{i=1}^H M^{[i-1]}w_{t-i}, 
\end{align}
where $H$ is a hyperparameter\footnote{In \ref{eq:DFC}, we give the form of a DFC, when the system is stable. We deal with the unstable case in Appendix \ref{appdx:sec:init-stable-policy}.}.

\par
Our main result is a polynomial-time algorithm for the case of known cost function, that achieves $n^3\sqrt{T}$-regret, where $n=d_x+d_u$. This is the optimal dependence in the time-horizon and our result improves upon the previous best known bound of $T^{2/3}$ \cite{hazan2019nonstochastic,simchowitz2020improper}. Perhaps more importantly than the regret bound is that, using ideas from convex geometry, we design a novel exploration strategy which significantly expands the existing algorithmic toolbox for balancing exploration and exploitation in linear dynamical systems, as we explain below.

\subsubsection*{Beyond explore-then-commit: the challenge of exploration}
 The only algorithms we know for this problem apply the simplest exploration strategy: \say{explore-then-commit} (ETC), known in control literature as certainty equivalence. In ETC, the learner spends the first $T_0$ steps playing random controls (e.g. $u_t\sim N(0,I)$), then estimates the system dynamics, and thereafter executes a greedy policy, based on these estimates. On the other hand, the whole stochastic bandit and RL theory literature is about sophisticated and sample-efficient exploration, mostly relying on the principle of \emph{optimism in the face of uncertainty} (OFU). Unfortunately, implementing OFU in online control requires solving optimization problems that are intractable in general \cite{abbasi2011regret}. Even though this computational issue can be circumvented for the case of quadratic costs using semidefinite programming \cite{cohen2019learning}, these techniques do not apply for general convex costs. In this work, we do not follow the OFU principle. Our exploration strategy is based on  adaptively constructing a sequence of barycentric spanners (Definition \ref{def:barycentric}) in the policy space.
 


\subsubsection*{The importance of general convex costs}
The special case of convex quadratic costs is the classical linear quadratic regulator and is frequently used, because it leads to a nice analytical solution when the system is known \cite{bertsekas1995dynamic}. However, this modeling choice is fairly restrictive, and in 1987, Tyrrell Rockafellar \cite{rockafellar1987linear} proposed the use of general convex functions for modelling the cost in a LDS, in order to handle constraints on state and control. In practice, imposing constraints is crucial for ensuring safe operating conditions.

\subsection{Statement of results.}
 We consider both the setting where $A_*$ is strongly stable (Assumption \ref{assump:stability}), and in the Appendix, we deal with unstable systems, by assuming that the  the learner is initially given a stabilizing linear policy (Assumption \ref{assump:init-stable-policy})\footnote{For unstable systems, without Assumption \ref{assump:init-stable-policy}, the regret is exponential in the $n$ (see \cite{chen2020black}).}. Our main result is the geometric exploration strategy given in Algorithm \ref{LJ-exploration}, for the case of known cost function. Algorithm \ref{alg:bandit-feedback} is for the case of bandit feedback. We now state informal versions of our theorems. 
Let $C=C(A_*,B_*, \Pi)$ denote a constant that depends polynomially on natural system and policy-class parameters.


\begin{theorem}[informal]\label{thm_inf:full-info}
For online control of LDS with known cost function, with high probability, Algorithm \ref{LJ-exploration} has regret \footnote{ $\widetilde{O}(1)$ hides logarithmic factors.}
\begin{align}
    R_T\leq \widetilde{O}(C)\cdot n^3 \sqrt{T}.
\end{align}
\end{theorem}

\begin{theorem}[informal]\label{thm_inf:bandit}
For online control of LDS with bandit feedback, with high probability, Algorithm \ref{alg:bandit-feedback} has regret
\begin{align}
    R_T\leq \widetilde{O}(C) \cdot poly(n) \sqrt{T}.
\end{align}
\end{theorem}

In Theorem \ref{thm_inf:bandit}, the polynomial dependence in $n$ is rather large ($n^{36}$). The large dimension dependence is typical in $\sqrt{T}$-regret algorithms for bandit convex optimization (BCO). Our setting is even more challenging than BCO, since the environment has a state.

\subsection{Prior Work}\label{subs:prior-work}
\paragraph{LQR:} When the cost $c$ is convex quadratic, we obtain the online linear quadratic regulator (LQR) \cite{abbasi2011regret,dean2018regret,mania2019certainty,cohen2019learning,simchowitz2020lqr}. The problem was introduced in \cite{abbasi2011regret}, and \cite{mania2019certainty,cohen2019learning,simchowitz2020lqr} gave $\sqrt{T}$-regret algorithms with polynomial runtime and polynomial regret dependence on relevant problem
parameters. In \cite{simchowitz2020lqr}, the authors proved that $\sqrt{T}$-regret is optimal.
\paragraph{Convex costs:}Closer to our work are recent papers on online control with general convex costs \cite{hazan2019nonstochastic,simchowitz2020improper,lale2020logarithmic,simchowitz2020making}. These papers consider even more general models, i.e, adversarially changing convex cost functions, adversarial disturbances \cite{hazan2019nonstochastic,simchowitz2020improper}, and \cite{simchowitz2020improper,lale2020logarithmic, simchowitz2020making} also address partial observation. Furthermore, all the considered policy classes can be expressed by DFCs. Despite the differences in the models, a common feature of these works is that all algorithms apply explore-then-commit (ETC). For our setting, ETC gives $T^{2/3}$-regret. Under the assumption that $c$ is strongly convex, the problem is significantly simplified and ETC achieves $\sqrt{T}$-regret \cite{simchowitz2020improper}. 
\paragraph{Linear system identification:} To address unknown systems, we make use of least-squares estimation \cite{simchowitz2018learning,sarkar2019near}. Recent papers deal with system indentification under partial observation \cite{oymak2019non, sarkar2019finite,tsiamis2019finite,simchowitz2019learning}.
\paragraph{Bandit feedback:} Control with bandit feedback has been studied in
\cite{cassel2020bandit} (known system) and \cite{gradu2020non} (both known and unknown system). Our result is comparable to \cite{gradu2020non}, and improves upon the $T^{3/4}$ regret bound that they achieve, when the disturbances are stochastic and the cost is a fixed function. 
 \paragraph{Barycentric spanners:}Barycentric spanners have been used for exploration in stochastic linear bandits \cite{awerbuch2008online}. However, in that context, the barycentric spanner is computed offline and remains fixed, while our algorithm adaptively changes it, based on the observed states.

\subsection{Paper Organization}

In the next section we give the notation and formally present our assumptions and the policy class. Section \ref{sec:warmup-case} is a warmup section, where we present the geometric exploration and its analysis for the simpler case where there are no dynamics ($A_*=0$). In Section \ref{sec:known-cost}, we describe our main contribution: the algorithm for the case of known cost function, and the proof of Theorem \ref{thm_inf:full-info}. Section \ref{sec:bandit-case} is devoted to the bandit case. Extensions of our results to even broader settings are described in Section \ref{sec:extensions}.
\section{Notation, assumptions and policy class}
\paragraph{Notation.}
For a matrix $A$, we use $\|A\|$ to denote its spectral norm, and for a positive semidefinite matrix $\Sigma \succcurlyeq 0$, we use $\|A\|_{\Sigma}$ to denote $\sqrt{\trace(A^T\Sigma A)}$. For notational simplicity, \say{with high probability} means with probability at least $1-1/T^{c}$, where $c$ is a sufficiently large constant.
\\

We define the class of strongly stable matrices. This definition was introduced in \cite{cohen2018online} and quantifies the classical notion of a stable matrix.

\begin{definition}
A matrix $A$ is $(\kappa,\gamma)$-strongly stable if there exists a decomposition of $A=Q\Lambda Q^{-1}$, where $\|\Lambda\|\leq 1-\gamma$ and $\|Q\|,\|Q^{-1}\|\leq \kappa$.
\end{definition}
For simplicity, in the main text we assume that $A_*$ is strongly stable. 
\begin{assumption}\label{assump:stability}
The system matrix $A_*$ is $(\kappa,\gamma)$-strongly stable, for some known constants $\kappa\geq 1$, $\gamma \geq 0$.
\end{assumption}
In Appendix \ref{appdx:sec:init-stable-policy}, we relax this assumption and consider possibly unstable $A_*$, by using an initial stabilizing policy (Assumption \ref{assump:init-stable-policy}). Assumption \ref{assump:init-stable-policy} is standard in online control literature (e.g., \cite{cohen2019learning, hazan2019nonstochastic,cassel2020logarithmic}), and without it the regret is exponential in $n$ (\cite{chen2020black}). In Appendix \ref{appdx:sec:init-stable-policy}, we will see that our results easily extend to this more general setting. The next assumptions are that $B_*$ is bounded and the cost $c$ is Lipschitz.
\begin{assumption}
The norm $\|B_*\|\leq \beta$, for some known constant $\beta\geq 1$.
\end{assumption}
\begin{assumption}
The cost function $c$ is $1$-Lipschitz. \footnote{We can easily account for more general $L$-Lipschitz costs via rescaling. Also, we can account for quadratic costs, by assuming Lipschitzness inside a ball where state and control belong with high probability.}
\end{assumption}
Finally, known cost function means that the algorithm has offline access to the value $c(x,u)$ and gradient $\nabla_{x,u}c(x,u)$ for all state/control pairs $(x,u)$.
Our policy class is all DFCs such that $\sum_{i=0}^{H-1} \left \|M^{[i]} \right\|\leq G$, for some $G\geq 1$, and we denote it by 
\begin{align}
   \mM=\left\{\left(M^{[0]},\dots,M^{[H-1]}\right) \ \Big|\ \sum_{i=0}^{H-1}\left\|M^{[i]}\right\|\leq G \right\}.
\end{align}
Finally, in accordance with previous works, we consider 
$H:= \widetilde{\Theta}(1) \cdot \gamma^{-1}$, where $\widetilde{\Theta}(1)$ denotes a large polylogarithmic factor\footnote{We make this choice for $H$, because this way DFCs can express all stabilizing LDCs (see \cite{simchowitz2020improper}). It is possible to let $H$ be a free parameter, which will add a polynomial in $H$ factor in our regret bounds.}.


\section{Warmup: $A_*=0$ and the hidden stochastic bandit problem}\label{sec:warmup-case}

We first demonstrate our exploration strategy for the special case where there are no dynamics, i.e., the matrix $A_*$ is zero. Even though this is significantly easier, it provides good intuition about the algorithm. We will also assume here that the cost depends only on the state: $c(x,u)=c(x)$, and the learner is only allowed to choose controls with norm $\|u\|\leq U$, for some known $U$ \footnote{We assume $U\geq 1$. This is without loss of generality, since $U$ serves as an upper bound.}. Since $A_*=0$, we have $x_{t+1}=B_*u_t+w_t$. Clearly, there is no point to consider policies here, since there is no dependence in the past. Thus, for $y_t:=x_{t+1}$, the natural regret is defined with respect to the best control, i.e.,
\begin{align}\label{eq:regret}
    R_T=\sum_{t=1}^T c(y_t) -T \min_{\|u\|\leq U} J(u),
\end{align}
where $J(u)=\mathbb{E}_{w\sim N(0,I)}[c(B_*u+w)]$. Observe that the problem we just defined is not exactly a special case of the initial one. However, we consider it here for the insights it offers. Let's see what the known techniques can achieve. First, if $c$ was a linear function, we could run LinUCB \cite{lattimore_szepesvari_2020} and get $\widetilde{O}(n\sqrt{T})$-regret. The difficulty is when $c$ is a general convex function. If $c$ is unknown, we can run the SBCO algorithm from \cite{agarwal2011stochastic}, which gives $\widetilde{O}(n^{33/2}\sqrt{T})$-regret. If $c$ is known, we can of course pretend we do not know it, and run the same algorithm to get the $\sqrt{T}$-dependence in the horizon. However, the dimension dependence is very large and the SBCO algorithm is very complicated. So, the interesting question is how to leverage the facts that 1) we know $c$, and 2) we observe $y_t$ (linear feedback), in order to achieve much better dependence in the dimension. This is a clean stochastic bandit problem, which to the best of our knowledge, has not been studied previously. We call it \say{SBCO with hidden linear transform}. In this section, we show how our geometric exploration, which is significantly simpler than the SBCO algorithms, achieves $\widetilde{O}(n^2\sqrt{T})$-regret. 


\subsubsection{Geometric Exploration}
We use the concept of barycentric spanners,
introduced in \cite{awerbuch2008online}. 
\begin{definition}\label{def:barycentric}
 Let $S$ be a compact set in $\mathbb{R}^n$. A set $V = \{v_1, v_2,\dots , v_n\}\subseteq S$ is a $C$-barycentric spanner for $S$ if every $v \in  S$ can be expressed as $v=\sum_{i=1}^n\lambda_iv_i$, where the coefficients $\lambda_i\in[-C,C]$.
\end{definition}
The power of a $C$-barycentric spanner is that if we know $B_*v_i$ up to $\ell_2$ error $\epsilon$, then we can infer $B_*v$, for any $v\in S$, up to error $Cn\epsilon$. Now, for constructing a $C$-barycentric spanner in polynomial time, it suffices to have access to a linear optimization oracle for $S$.
\begin{theorem}[Proposition 2.5 in \cite{awerbuch2008online}]\label{thm:lin-opt-oracle-simple}
Suppose $S\subseteq \mathbb{R}^n$ is compact and not contained in any proper linear subspace.  Given an oracle for optimizing linear functions over $S$, for any $C > 1$ we can compute a $C$-barycentric spanner for $S$ in polynomial time, using $O(n^2 \log_C(n))$ calls to the oracle.
\end{theorem}
Now, we are ready to describe our algorithm (Algorithm \ref{geometric-exploration-warmup}). It runs in epochs and follows the phased-elimination paradigm \cite{lattimore_szepesvari_2020}. During epoch $r$, it focuses on a convex set of controls $\mathcal{U}_r$, which by the end of the epoch will be substituted by $\mathcal{U}_{r+1}\subseteq \mathcal{U}_r$, shrinking towards the optimal control. Roughly, $\mathcal{U}_r$ can be thought as a sublevel set of $J(u\ |\ \widehat{B}):=\mathbb{E}_{w\sim N(0,I)}\left[c\left(\widehat{B}u+w\right)\right]$, where $\widehat{B}$ is an estimate of $B_*$ constructed using observations from past epochs. To eliminate suboptimal policies from $\mathcal{U}_r$ and get $\mU_{r+1}$, it suffices to refine $\whB$ only in the directions relevant to the controls of $\mU_r$. To do this, in the beginning of epoch $r$, the algorithm constructs a barycentric spanner of $\mathcal{U}_r$. As we explain in Subsection \ref{subs:poly-time-warmup}, this can be done in polynomial time, because we know the cost function $c$ and $\mU_r$ is convex. The elements of the spanner are the exploratory controls that we execute during the epoch and lead to generation of observations that give the information needed to refine $\whB$. 

\SetAlgoNoLine
\SetKwInput{KwInput}{Input}              
\SetKwInput{KwOutput}{Output}  
\begin{algorithm}[H]  
 \textbf{Input:} 
 Initialize $U_1=\{u\in \mathbb{R}^n\ |\ \|u\|\leq U\}$.\\
 Set $t=1$.\\
  \For{$r=1,2,\dots$}{
    Set $\epsilon_r=2^{-r}$.\\
    Compute a 2-barycentric spanner of $\mU_{r}$: $\{u_{r,1},u_{r,2},\dots,u_{r,n}\}$.\
    
    
    Set $T_r=\widetilde{\Theta}(1)\cdot \epsilon_r^{-2}\cdot n^2(n+\beta^2)$ ($\widetilde{\Theta}(1)$ denotes a large  polylogarithmic factor).
    \\
    \For{$j=1,2,\dots,n$}{
         Apply control $u_{r,j}$ for $T_r$ steps.\\
         Set $t=t+T_r$.
    }
    Let $\widehat{B}_r$ be a minimizer of 
  $\sum_{s=1}^{t-1}\left\|B u_s -y_{s}\right\|^2 + \|B\|_F^2 $,
  over all $B$.\\ 
    Eliminate suboptimal policies:
    \begin{align}\label{eq:elimination_simple}
    \mU_{r+1}=\left\{u\in \mU_r\ \Bigg{|}\ J\left(u \  \Big{|}\ \whB_{r} \right)-\min_{u'\in \mU_r}J\left(u' \  \Big{|}\ \whB_{r} \right)\leq 3\epsilon_r\right\}
    \end{align}
    \tcp{Equation \ref{eq:elimination_simple} is a definition, not a step that requires computation.}
  }
  \caption{Geometric Exploration for SBCO with hidden linear transform}
  \label{geometric-exploration-warmup}
\end{algorithm}

The following theorem bounds the regret incurred by Algorithm \ref{geometric-exploration-warmup}.
\begin{theorem}\label{thm:geometric-exploration-warmup}
With high probability, Algorithm \ref{geometric-exploration-warmup} achieves regret $R_T\leq \widetilde{O}(1)\cdot \sqrt{\beta^2U^2 n^3(n+\beta^2)T}.$
\end{theorem}

In the next subsection, we present the proof. The main new ideas are in the general case, rather than this simpler setting. In the general case, the high-level proof structure will remain the same, but we will have to add new ideas in order to make it work. After the proof, we will explain why Algorithm \ref{geometric-exploration-warmup} can be implemented to run in polynomial time.

\subsection{Proof}
The main component of the proof is bounding the \emph{average regret}, i.e.,
\begin{align}
    R_T^{avg}:=\sum_{t=1}^TJ (u_t )-T\cdot J(u_*),
\end{align}
where $u_*\in \argmin_{\|u\|\leq U}J(u)$. We first show that $R_T^{avg}$ upper bounds $R_T$ up to an $\widetilde{O}(\sqrt{T})$ error. 

\begin{lemma}\label{lem:martingale-warmup}
With high probability, $R_T\leq R_T^{avg}+\widetilde{O}(\sqrt{T})$.
\end{lemma}
\begin{proof}
    We use Lipschitz concentration, followed by Azuma's inequality. We define the filtration $\mF_t=\sigma(w_1,w_2,\dots,w_{t-1})$, and for a given $u$, let $f_u(w)=c(B_*u+w)$. Clearly, $f_u(w)$ is a 1-Lipschitz function of $w$. So, conditioned on $\mF_t$, $f_{u_t}(w_t)$ is $O(1)$-subgaussian (from Gaussian concentration \cite{vershynin2018high}). Thus, from Azuma's inequality, we have that with high probability,
\begin{align}
    R_T-R_T^{avg}=\sum_{t=1}^T\left( f_{u_t}(w_t)- \mathbb{E}_w \left[f_{u_t}(w)\right]\right)\leq \widetilde{O}(\sqrt{T}).
\end{align}
\end{proof}
We will now show that with high probability, 
\begin{align}\label{eq:target-bound-avg-regret}
    R_T^{avg}\leq  \widetilde{O}(1)\cdot \sqrt{\beta^2U^2 n^3(n+\beta^2)T},
\end{align}
which will conclude the overall proof. To do so, we will show that with high probability, for all controls $u\in \mU_r$, the suboptimality gap $R^{avg}(u):=J(u)-J(u_*)\leq O(2^{-r})$, from which the bound \ref{eq:target-bound-avg-regret} follows after some calculations (using the fact that the controls executed during epoch $r$ belong to $\mU_r$). To bound the suboptimality gap, we prove that the elimination step of Algorithm \ref{geometric-exploration-warmup} (Equation \ref{eq:elimination_simple}) is effective, i.e., it removes only the $\Omega(2^{-r})$-suboptimal controls ($R^{avg}(u)\geq \Omega(2^{-r})$). This effectiveness is guaranteed because as we will show, the estimation error $ | J(u\ |\  \whB_{r})-J(u\ |\ B_*)|\leq2^{-r} $, for all $u\in \mU_r$. We now formally implement these steps, starting from the most important.

\begin{lemma}\label{lem:cost-estimation-warmup}
With high probability, for all epochs $r$ and for all $u\in \mU_r$, 
\begin{align}\label{eq:cost-estimation-warmup}
    \left|J(u \  | \ \whB_{r} )-J(u \ | \ \whB_{*} )\right|\leq 2^{-r}.
\end{align}
\end{lemma}

\begin{proof}
We have  $\left|J(u \  | \ \whB_{r} )-J(u \ | \ B_{*} )\right|= \left|\mathbb{E}_w[c(\whB u +w)-c(B_*u+w)]\right|\leq \|( \whB_{r}-B_{*})u\|. $ Let $\Delta_r=\whB_{r}-B_{*}$. Since $u\in \mU_r$, there exist coefficients $\lambda_i\in[-2,2]$, such that $u=\sum_{i=1}^n\lambda_i u_{r,i}$. So,
\begin{align}
  \|\Delta_ru\|^2= \left\| \sum_{i=1}^n\lambda_i \Delta_r u_{r,i} \right\|^2 &\leq \left( \sum_{i=1}^n\lambda_i^2\right)\sum_{i=1}^n \|\Delta_r u_{r,i}\|^2 \notag\\
  & \leq 4n\cdot \sum_{i=1}^n \|\Delta_r u_{r,i}\|^2.
\end{align}
We will show that $\sum_{i=1}^n \|\Delta_r u_{r,i}\|^2$ is small, using the fact that $u_{r,i}$ are the controls that the algorithm plays, so $\Delta_r$ must be small in these directions. Formally, we have the following claim.
\begin{claim}
With high probability, for all epochs $r$, 
\begin{align}
    \sum_{i=1}^n \|\Delta_r u_{r,i}\|^2\leq \widetilde{O}(1)\cdot\frac{n(n+\beta^2)}{T_r}
\end{align}
\end{claim}
Observe that our choice of $T_r$ makes the above bound at most $2^{-2r}/(C\cdot n)$, for some large constant $C$. Thus, we get that with high probability, for all $r$ and $u\in \mU_r$, the error $\|\Delta_ru\|$ is at most $2^{-r}$, which finishes the proof of Lemma \ref{lem:cost-estimation-warmup}.
\end{proof}
We now prove the claim.
\begin{proof}
Let $V_t=\sum_{s=1}^{t-1}u_su_s^T +I$ and $t_r$ be the timestep $t$ when we compute $\whB_r$. Now, $\whB_r$ is computed via Least-Squares, so it satisfies the following guarantee (see \cite{cohen2019learning}, Lemma 6): with high probability, for all $r$,
\begin{align}\label{eq:LS-guarantee}
    \|\Delta_r^T\|_{V_{t_r}}^2 \leq \widetilde{O}(n)\cdot \log{\det(V_{t_r}}) + O(1)\cdot \|B_*\|_F^2.
\end{align}
From AM-GM inequality, we get $(\det(V_{t_r}))^{1/n}\leq tr(V_{t_r})/n\leq 1+ 1/n\cdot \sum_{s=1}^{t_r-1}\|u_s\|^2\leq 1+(t_r-1)U^2/n$. Thus, $\log{\det(V_{t_r}})\leq \widetilde{O}(n)$. Combining with \ref{eq:LS-guarantee}, we get that with high probability, for all $r$, $ \|\Delta_r^T\|_{V_{t_r}}^2\leq \widetilde{O}(n^2+n\beta^2) $. The fact that $\|\Delta_r^T\|_{V_{t_r}}^2\geq T_r\cdot \sum_{i=1}^n\|\Delta_ru_{r,i}\|^2$ finishes the proof.
\end{proof}

Now, we bound $R^{avg}(u)$, for all $u\in \mU_r$, and show that all $U_r$ contain $u_*$.
 \begin{lemma}\label{lem-main:regr_each_step-simple}
 With high probability, for all $r$, we have
\begin{itemize}
\item $u_*\in \mU_r$, and
\item for all $u\in \mU_{r+1}$, $R^{avg}(u)\leq 5\cdot2^{-r}$. 
\end{itemize}
\end{lemma}
\begin{proof}
We condition on the event that inequality \ref{eq:cost-estimation-warmup} holds for all $r$ and $u\in \mU_r$. For the first bullet of the lemma, suppose that for some $r$, we have $u_*\in \mU_r$ and $u_*\notin \mU_{r+1}$. Thus, there exists $u\in \mU_r$, such that $J(u\ |\ \whB_{r})<J(u_*\ |\ \whB_{r})-3\epsilon_r$. Then, inequality \ref{eq:cost-estimation-warmup} implies that $J(u\ |\ B_*)-\epsilon_r<J(u_*\ |\ B_*)+\epsilon_r-3
\epsilon_r$,
which contradicts the optimality of $u_*$. 
\par For the second bullet, if $u\in \mU_{r+1}$, then $J(u\ |\ \whB_{r})-J(u_*\ |\  \whB_{r})\leq 3\epsilon_r$, because we showed that $u_*\in \mU_r$. By applying inequality \ref{eq:cost-estimation-warmup}, we get $J(u\ |\ B_*)-J(u_*\ |\  B_*)\leq 5\epsilon_r$.
\end{proof}

Now, we are ready to finish the proof of the  theorem. We have
 $R_T^{avg}=\sum_{t=1}^TR^{avg}(u_t)=\sum_{r=1}^{q}\sum_{j=1}^n T_{r}\cdot R^{avg}(u_{r,j})$, where $q$ is the total number of epochs. Also, $u_{r,j}\in \mU_r$, because it is an element of the barycentric spanner. So, the previous lemma implies that with high probability, for all $r\geq 2$ and for all $j$, we have $R^{avg}(u_{r,j})\leq 5\cdot 2^{-(r-1)}$. We now bound $\sum_{r=1}^q2^{r}$. Observe that $T_r=   D\cdot 2^{2r}$, where $D=\widetilde{\Theta}(1) \cdot n^2(n+\beta^2)$ and $T\geq 1/2\cdot \sum_{r=1}^q n\cdot T_r= 1/2\cdot nD \sum_{r=1}^q2^{2r}$ (the $1/2$ is because the horizon can end before the end of the final epoch). Thus, $\sqrt{\frac{2T}{Dn}}\geq  \sqrt{\sum_{r=1}^q2^{2r}}\geq q^{-1/2}\sum_{r=1}^q2^{r}$, by Cauchy-Schwarz. Using that $q\leq O\left(\log{T}\right)$, we get $\sum_{r=1}^q2^{r}\leq \widetilde{O}(1)\cdot \sqrt{\frac{T}{Dn}}$. Summarizing, by excluding the first epoch, we have
\begin{align}
  R_T^{avg}-\sum_{j=1}^{n} T_{1}\cdot R^{avg}(u_{1,j})\leq  \sum_{r=2}^{q}\sum_{j=1}^n D \cdot2^{2r}\cdot 5\cdot2^{-(r-1)} & \leq O(1)\cdot Dn\cdot \sum_{r=1}^q2^{r} \notag \\
  &\leq \widetilde{O}(1)\cdot \sqrt{D\cdot n \cdot T}.
\end{align}
Also, for any $u\in \mU_1$, $R^{avg}=\mathbb{E}_w[c(B_*u+w)-c(B_*u_*+w)]\leq \|B_*(u-u_*)\|\leq 2\beta U$. Now, because the horizon $T$ can be shorter than the length of the first epoch $(nT_1)$, we have 
\begin{align}
    R_T^{avg}\leq 2\beta U\cdot \min(nT_1,T)+ \sqrt{D\cdot n \cdot T}& \leq 2\beta U\cdot \sqrt{nT_1T}+ \sqrt{D\cdot n \cdot T} \notag \\
    & \leq \widetilde{O}(1)\cdot \sqrt{\beta^2U^2 n^3(n+\beta^2)T}.
\end{align}

\subsection{Polynomial running time}\label{subs:poly-time-warmup}
It suffices to argue that for each epoch $r$, a 2-barycentric spanner of $\mU_r$ can be computed in polynomial time. First, the conditions of Theorem \ref{thm:lin-opt-oracle-simple} are satisfied: 1) $\mU_r$ is compact, since $c$ is Lipschitz and so $J(\cdot\ |\ \whB_r)$ is continuous, and 2) from the proof of Lemma \ref{lem-main:regr_each_step-simple} we can see that not only $u_*\in \mU_r$, but also there exists a small ball around $u_*$ which is contained in $\mU_r$, so $\mU_r$ is not contained in any proper linear subspace. Thus, to prove polynomial time, it suffices to have a linear optimization oracle for $\mU_r$. Observe that $\mU_r$ is convex, being an intersection of sublevel sets of convex functions. So, given a separation oracle for $\mU_r$, a linear optimization oracle can be implemented in polynomial time via the ellipsoid method. Now, such a separation oracle can be implemented, given access to $J(u\ |\ \whB_{r'})$ and $\nabla_u J(u\ |\ \whB_{r'})$ for all controls $u$ and epochs $r'< r$. Even though we don't have exact access to these quantities because of the expectations they involve, we can approximate them in polynomial time up to $1/\poly(T)$-error by averaging samples, since we have offline access to the values $c(x)$ and gradients $\nabla_{x}c(x)$, for any $x$. Folklore approximation arguments suffice to show that even with this $1/\poly(T)$-error, ellipsoid method can optimize linear functions up to $1.01$ multiplicative error, in polynomial time. Now, just by inspecting the algorithm from \cite{awerbuch2008online} which is used to prove Theorem \ref{thm:lin-opt-oracle-simple}, we can see that even with a $1.01$-approximate optimization oracle (instead of an exact one) the same proof goes through. Thus, a 2-barycentric spanner of $\mU_r$ can be constructed in polynomial time.
\section{General case for known cost function}\label{sec:known-cost}
In this section, we extend the algorithm and the proof to tackle the general case of a $(\kappa,\gamma)$-strongly stable $A_*$, when the cost function is known. Here, the algorithm seeks the optimal policy (instead of control), and the controls applied at some step affect the states at future steps. First, we give some background on DFCs.

\subsection{Disturbance-Feedback-Control policies}\label{subs:disturbance-policies}
DFCs apply controls
\begin{align}\label{eq:noise-based-policy}
  u_t=\sum_{i=1}^H M^{[i-1]}w_{t-i}.
\end{align}
This parameterization has the advantage that the infinite horizon cost $J(M)$ of the policy $M=\left(M^{[0]},\dots,M^{[H-1]}\right)$ is a convex function of $M$ \footnote{This is not true for the smaller class of stabilizing LDCs.}.
Under the execution of $M$, the state can be expressed as $x_{t+1}=A_*^{H+1}x_{t-H}+\sum_{i=0}^{2H}\Psi_i(M\ |\ A_*,B_*)w_{t-i}$ \footnote{We set $x_s=0$, for $s\leq 0$.}, where $\Psi_i(M\ |\ A_*,B_*)$ are affine functions of $M$. We provide exact expressions for $\Psi_i$ in Appendix \ref{appdx:disturbance-based}. Because $A_*$ is $(\kappa,\gamma)$-strongly stable, it can be shown that $A_*^{H+1}\approx 0$. This leads to the definition of two time-independent quantities: surrogate state and surrogate control. 
\paragraph{Surrogate state and control:}\label{par:sur-state-control}
 Let $\eta=(\eta_0,\eta_1,\dots \eta_{2H})$, where  $\eta_i\simiid N(0,I)$. We define
$u\big(M\ |\ \eta\big)=\sum_{i=0}^{H-1}M^{[i]}\eta_i$ and $ x\big(M\ |\ A_*,B_*,\eta \big)=\sum_{i=0}^{2H}\Psi_i(M\ |\ A_*,B_*)\eta_i$. As we mentioned, $A_*^{H+1}\approx 0$, so when we execute policy $M$ and $t\geq \Omega(H)$, the state/control pair $(x_t,u_t)$ is almost identically distributed with $\left(x\big(M\ |\ A_*,B_*,\eta \big), u\big(M\ |\ \eta\big)\right)$. 
\paragraph{Convex surrogate cost:} 
 We now define a function that approximates the cost $J(M)$, without involving infinite limits that create computational issues. Let 
 \begin{align}
\mathcal{C}\big(M\   |\ A_*, B_* \big):=\mathbb{E}_{\eta} \Bigg[ c\Big(x(M\ |\ A_*,B_*,\eta),\ u(M\ |\ \eta )\Big) \Bigg].     
 \end{align}
The cost $\mathcal{C}$ is convex, because $x$ and $u$ are affine in $M$ and $c$ is convex. 
The following theorem establishes that $J(M)$ is almost equal to $\mathcal{C}\big(M\   |\ A_*, B_* \big)$.

\begin{theorem}[\cite{agarwal2019online}]\label{thm:truncation}
For all $M\in \mathcal{M}$, we have $\Big|\mathcal{C}\left(M\   |\ A_*, B_* \right)-J(M)\Big|\leq 1/T$.
\end{theorem}
This theorem is almost proved in \cite{agarwal2019online}. For completeness, we provide its proof in Appendix \ref{apdx:prf:truncation}.
The algorithms we present in the paper aim to minimize $\mathcal{C}\left(M\   |\ A_*, B_* \right)$. The difficulty is that we do not know this function, since we do not know $A_*,B_*$.

\subsection{Affine barycentric spanners}
Before we present the algorithm, we will need to slightly modify Definition \ref{def:barycentric} and Theorem \ref{thm:lin-opt-oracle-simple}, to take into account that the state is an \emph{affine} function of the policy (instead of just linear).
\begin{definition}\label{def:affine-barycentric}
 Let $S$ be a compact set in $\mathbb{R}^d$. A set $V = \{v_0,v_1, \dots , v_d\}\subseteq S$ is an affine $C$-barycentric spanner for $S$ if every $v \in  S$ can be expressed as $v=v_0+\sum_{i=1}^d\lambda_i(v_i-v_0)$, where the coefficients $\lambda_i\in[-C,C]$.
\end{definition}
Here, $S$ will be a set of policies, so the dimension will be $d=d_xd_uH$.
In terms of computation, in Appendix \ref{appdx:prf:thm:lin-opt-oracle-affine}, we show that Theorem \ref{thm:lin-opt-oracle-simple} holds almost unchanged for this case. The only modification is in one of its technical conditions, which can be easily satisfied in our cases of interest.
\begin{theorem}\label{thm:lin-opt-oracle-affine}
Suppose $S\subseteq \mathbb{R}^d$ is compact and not contained in any proper affine subspace.  Given an oracle for optimizing linear functions over $S$, for any $C > 1$ we can compute an affine $C$-barycentric spanner for $S$ in polynomial time, using $O(d^2 \log_C(d))$ calls to the oracle.
\end{theorem}

\subsection{Algorithm and main result} 
\par  Algorithm \ref{LJ-exploration} receives as input some (rough) initial estimates $A_0 , B_0$ that approximate the true system matrices $A_* , B_*$ within error $\epsilon$. As we state in Theorem \ref{thm:LJ-exploration}, $\epsilon$ needs to be $1/\poly(d_x,d_u)$, and we can make sure this is satisfied by executing a standard warmup exploration procedure, given in Appendix \ref{appdx_subs:warm-up}.

Observe that the algorithm is a straightforward generalization of Algorithm \ref{geometric-exploration-warmup} (although its analysis will not be). There are two conceptual differences. First, the policy $M_{r,0}$ is executed $d$ times more that the other elements of the barycentric spanner. The reason is that in Definition \ref{def:affine-barycentric}, $v_0$ has $O(d)$ times more weight than the other elements, once we write $v$ as a linear combination of $\{v_0,v_1,\dots,v_d\}$. Second, since we do not know the disturbances, we cannot execute a policy $M$. Thus, we compute estimates $\whw_t$, and we use these instead. We now formally state our main theorem.

\begin{theorem}\label{thm:LJ-exploration}
   Let $C_1=\kappa^{4}\beta^2 \gamma^{-5}G^2$, $C_2=\kappa^2\beta\gamma^{-1/2}$, $C_3=\kappa^{4}\gamma^{-2}$, $C_4=\kappa^{2}\beta \gamma^{-2}G$ and $C_5=\kappa^{8}\beta^{2}\gamma^{-5}G$. Suppose that the initial estimation error bound $\|(A_0 \ B_0)-(A_* \ B_*)\|_F\leq \epsilon$ satisfies $ \epsilon^2 \leq \left(C_1\cdot d_xd_u(d_x+d_u)\right)^{-1}$, and the initial state has norm $\|x_1\|\leq \tO
   (C_2)\cdot \sqrt{d_x}$. 
 Assume $ T\geq \widetilde{\Omega}(C_3)$. Then, with high probability, Algorithm \ref{LJ-exploration} satisfies
\begin{align}
    R_T\leq \tO(C_4)\cdot d_xd_u(d_x+d_u)\sqrt{
T}+ \tO(C_5)\cdot (d_x+d_u)^{6.5}.
\end{align}
\end{theorem}
In \cite{cohen2019learning}, the authors analyze the warmup exploration (Algorithm \ref{alg:warm-up} in Appendix  \ref{appdx_subs:warm-up}). In Appendix \ref{appdx_subs:warm-up}, we show that their analysis can be combined with Theorem \ref{thm:LJ-exploration} to show the following. 

\begin{corollary}\label{cor:main_thm}
Let $C_3=\kappa^{4}\gamma^{-2}$, $C_4=\kappa^{5}\beta^2 \gamma^{-2.5}$ and $C_6=\kappa^{8}\beta^{3}\gamma^{-6}G^3$. Assume $ T\geq \widetilde{\Omega}(C_3)$. If we run the warmup exploration and then run Algorithm \ref{LJ-exploration}, then with high probability, 
\begin{align}
    R_T\leq \tO(C_4)\cdot d_xd_u(d_x+d_u)\sqrt{
T}+ \tO(C_6)\cdot (d_x+d_u)^{6.5}
\end{align}
\end{corollary}
In the next subsection, we explain the key steps that need to be added to the proof of Section \ref{sec:warmup-case}, to prove Theorem \ref{thm:LJ-exploration}, while we defer the details to the Appendix. The argument about polynomial time computation is essentially identical with the one in Section \ref{sec:warmup-case}, so we do not repeat it.

\SetAlgoNoLine

\SetKwInput{KwInput}{Input}              
\SetKwInput{KwOutput}{Output}  
\begin{algorithm}[H]  
 \textbf{Input:} estimates $A_0 , B_0$ satisfying $\|(A_0 \ B_0)-(A_* \ B_*)\|_F\leq \epsilon$.\\
 Initialize policy set $\mM_1=\mM$, matrix estimates $(\wh{A}_1\ \wh{B}_1) =(A_0\ B_0)$, disturbance estimates $\widehat{w}_t=0$ for $t\leq 0$.\\
  Set $d=d_u d_x H$.\\
  Set $t=1$ and observe $x_1$ (state $x_1$ is essentially the state after the end of the warmup exploration).\\
  \For{$r=1,2,\dots$}{
    Set $\epsilon_r=2^{-r}$.\

    Compute an affine 2-barycentric spanner of $\mM_{r}$: $\{M_{r,0},M_{r,1},\dots,M_{r,d}\}$.\
    
    
    Set $T_r=\widetilde{\Theta}( \kappa^4\gamma^{-3})\cdot \epsilon_r^{-2}\cdot d_xd_u(d_x+d_u)^2$.\\
    Call Execute-Policy$\left(M_{r,0},\ d\cdot T_r\right)$.\\
    \For{$j=1,\dots,d$}{
         Call Execute-Policy$\left(M_{r,j},\ T_r\right)$.
    }
    
    Set $t_r=t$ (current timestep).\
    
    Eliminate suboptimal policies: 
    \begin{align}\label{eq:elimination}
    \mM_{r+1}=\left\{M\in \mM_r\ \Bigg{|}\ \mC\left(M \  \Big{|}\ \whA_{t_r}, \whB_{t_r} \right)-\min_{M'\in \mM_r}\mC\left(M' \  \Big{|}\ \whA_{t_r}, \whB_{t_r} \right)\leq 3\epsilon_r\right\}
    \end{align}
  }
  \caption{Geometric Exploration for Control}
  \label{LJ-exploration}
\end{algorithm}

\begin{algorithm}[H]\label{Execute}
  \KwInput{Policy $M$ and execution-length $L$.}
  \For{$s=1,2,\dots,L$}{
       \  Apply control $u_t=\sum_{i=1}^H M^{[i- 1]}\widehat{w}_{t-i}$.\\
        Observe $x_{t+1}$.\\
        Call System-Estimation, to get $\widehat{A}_{t+1},\widehat{B}_{t+1}$.\\
        Record the estimate $\widehat{w}_t=x_{t+1}-\widehat{A}_{t+1} x_t -\widehat{B}_{t+1} u_{t}$.\\
        Set $t=t+1$.
    }
  \caption{Execute-Policy}
  \label{alg:execute}
\end{algorithm}

\begin{algorithm}[H]
 Set $\lambda=\widetilde{\Theta}\left(\kappa^{4}\beta^2\gamma^{-5}G^2\right) \cdot d_xd_u(d_x+d_u)^3$.\\ 
  Let $(\whA_{t+1}\ \whB_{t+1})$ be a minimizer of 
  $$\sum_{s=1}^t\left\|(A\ B)z_s -x_{s+1}\right\|^2 +\lambda \|(A\  B) -(A_0\ B_0)\|_F^2 $$
  over all $(A\ B)$, where $z_s=\begin{pmatrix} x_s \\ u_s\end{pmatrix}$.
  \caption{System-Estimation }
  \label{alg:LS}
\end{algorithm}

\subsection{Proof: the key steps}\label{sub:overview-full-cost}
\subsubsection*{Estimating disturbances requires no exploration}
The first key step is proving that $(\whw_t)_t$ are on average very accurate (for large enough $T$).
\begin{lemma}\label{lem-main:noise-approx}
With high probability, we have $\sum_{t=1}^T \|\widehat{w}_t-w_t\|^2\leq \widetilde{O}(1)\cdot (d_x+d_u)^3.$
\end{lemma}
This lemma will alleviate all the issues arising from the fact that we do not know the actual disturbances. An important aspect of its proof will be that we will not use any information about the directions of the controls $u_t$. On the other hand, the choice of controls matters for estimating $A_*,B_*$ (e.g., if we constantly play $u_t=0$, then we get no information about $B_*$). Thus, the disturbances can be accurately estimated without accurately estimating $A_*,B_*$, in other words, \emph{without exploring}.
\begin{proof}
We define $\Delta_t:=(\whA_t\ \whB_t)- (A_*\ B_*)$. We show the following lemma, using the fact that the system estimation is done via Least-Squares, and that the choice of controls depends on this estimation.
\begin{lemma}\label{lem:LS}
 Let $V_t=\sum_{s=1}^{t-1}z_sz_s^T+\lambda \cdot I$. Then, with high probability, for all t, 
    \begin{align}
        \left \|\Delta_{t}^T \right\|_{V_t} \leq \widetilde{O}(d_x+d_u),
    \end{align}
and $\|z_t\|\leq \tO(\kappa^2\beta\gamma^{-1}G)\cdot \sqrt{d_x}$.
\end{lemma}
Similar bounds appear in \cite{abbasi2011regret} and \cite{cohen2019learning}, and because the proof technique is standard, we defer it to the Appendix \ref{appdx_subs:LS}. Now, we have
\begin{align}
    \left\|\widehat{w}_t-w_t \right\|^2&=\left\|x_{t+1}-\widehat{A}_{t+1} x_t -\widehat{B}_{t+1} u_{t}-\left(x_{t+1}-A_* x_t -B_* u_{t}\right)\right\|^2 \notag \\
    &=\left\|\Delta_{t+1}z_t \right\|^2 \leq \left\|\Delta_{t+1}V_{t+1}^{\frac{1}{2}}\right \|^2\cdot \left\|V_{t+1}^{-\frac{1}{2}}z_t \right\|^2. \notag
\end{align}
Lemma \ref{lem:LS} implies that with high probability $\left\|\Delta_{t+1}V_{t+1}^{\frac{1}{2}}\right \|^2\leq \widetilde{O}(1)\cdot(d_x+d_u)^2$, for all $t$. The bound on $\sum_{t=1}^T\left\|V_{t+1}^{-\frac{1}{2}}z_t \right\|^2=\sum_{t=1}^T\left\|z_t \right\|_{V_{t+1}^{-1}}^2$ follows from a linear-algebraic inequality, which has previously appeared in the context of online optimization \cite{hazan2007logarithmic} and stochastic linear bandits \cite{dani2008stochastic}. 
\begin{lemma}[\cite{lattimore_szepesvari_2020}]
Let $V_0$ positive definite and $V_{t}=V_0+\sum_{s=1}^{t-1}z_sz_s^T$, where $z_1,\dots,z_T\in \mathbb{R}^n$ is a sequence of vectors with $\|z_t\|\leq L$, for all $t$. Then,
\begin{align}
    \sum_{t=1}^T\min\left (1, \|z_t\|_{V_{t+1}^{-1}}^2\right)\leq 2 n \log\left(\frac{\trace(V_0)+TL^2}{n \det^{1/n}(V_0)}\right).
\end{align}
\end{lemma}
Combining with the bound on $\|z_t\|$ from Lemma \ref{lem:LS}, we get that with high probability, \begin{align}
    \sum_{t=1}^T\min\left (1, \|z_t\|_{V_{t+1}^{-1}}^2\right)\leq \widetilde{O}(d_x+d_u).
\end{align}
The bound on $\|z_t\|$ and the fact that $\lambda I\preccurlyeq V_{t+1}$ imply that with high probability, $ \|z_t\|_{V_{t+1}^{-1}}^2\leq 1$, which finishes the proof.
\end{proof}
We now move to the second key step of the proof.
\subsubsection*{The coupling argument}
Eventually, we will prove the analog of Lemma \ref{lem:cost-estimation-warmup} of Section \ref{sec:warmup-case}, i.e, that with high probability, for all epochs $r$ and for all $M\in \mM_r$, we have $\left|\mC(M\ |\ \whA_{t_r} , \whB_{t_r})-\mC(M\ |\ A_*,B_*)\right|\leq 2^{-r}$. In the proof of Lemma \ref{lem:cost-estimation-warmup}, we easily bounded $\left|J(u \  | \ \whB_{r} )-J(u \ | \ B_{*} )\right|$ by $ \|( \whB_{r}-B_{*})u\|$: the estimation error on the cost by the estimation error on the matrix (in the direction of $u$). The second key step of the proof of the general case is proving an analog of this bound. The reason that this is challenging is that for some matrices $A$ and $B$, the cost $\mathcal{C}(M\ |\ A,B)$ is nonlinear in $A$ and $B$. To state the bound, we will need some definitions.
Let $\Sigma(M)$ be the covariance matrix of the random vector $z\left(M\ |\ A_*,B_*,\eta\right)$ defined as 
\begin{align}\label{eq:def:z(M)}
  z\left(M\ |\ A_*,B_*,\eta\right):= \begin{pmatrix}x\left(M\ | \ A_*,B_*,\eta\right)\\ u\left(M\ |\ \eta\right)\end{pmatrix},  
\end{align}
where $x\left(M\ | \ A_*,B_*,\eta\right)$, $u(M\ |\ \eta)$ and the distribution of $\eta$ are given in Subsection \ref{par:sur-state-control}. In other words, $\Sigma(M)=\mathbb{E}_{\eta}\left[z\left(M\ |\ A_*,B_*,\eta\right)\cdot z\left(M\ |\ A_*,B_*,\eta \right)^T \right]$. We prove the following lemma.
\begin{lemma}\label{lem:reduction-cost-matrices}
Let $\whA,\whB$ be estimates of $A_*,B_*$, $\Delta= (\whA\ \whB)-(A_*\ B_*)$, and $\|\Delta\|\leq \frac{\gamma}{2\kappa^2}$. For all $M\in \mM$, we have
\begin{align}\label{eq:reduction}
    \left|\mC(M\ | \ \whA , \whB)-\mC(M\ |\ A_*,B_*)\right|\leq 6\kappa^2\gamma^{-1}\Big(\left\|\Delta^T \right\|_{\Sigma(M)}+1/T\Big).
\end{align}
\end{lemma}
To prove the lemma, we employ a probabilistic coupling argument between the true system and the estimated one. This could be a useful technique for future works.
\begin{proof}
 We fix a policy $M\in \mM$. 
\begin{align}\label{eq:target:reduction}
   & \left|\mC(M\ |\ \whA , \whB)-\mC(M\ |\ A_*,B_*)\right| \notag \\ &=\Bigg{|}\mathbb{E}_{\eta} \bigg[ c\Big(x(M\ |\ \whA,\whB,\eta),\ u(M\ |\ \eta )\Big) \bigg]  -\mathbb{E}_{\eta} \bigg[ c\Big(x(M\ |\ A_*,B_*,\eta),\ u(M\ |\ \eta )\Big) \bigg] \Bigg{|}\notag \\
   &\leq \mathbb{E}_{\eta}\bigg{|}c\Big(x(M\ |\ \whA,\whB,\eta),\ u(M\ |\ \eta )\Big)-c\Big(x(M\ |\ A_*,B_*,\eta),\ u(M\ |\ \eta )\Big)\bigg{|} \notag \\
   &\leq \mathbb{E}_{\eta}\left \|x(M\ |\ \whA,\whB,\eta)-x(M\ |\ A_*,B_*,\eta)\right\|,
\end{align}
where we used the fact that $c$ is 1-Lipschitz.

\par To bound \ref{eq:target:reduction}, we create two coupled dynamical systems: $x_{1}^{(1)}=x_{1}^{(2)}=0$,
\begin{align}
    x_{t+1}^{(1)}=\whA x_t^{(1)}+\whB u_t+w_t\ \ \ \text{and}\ \ \ x_{t+1}^{(2)}=A_*x_t^{(2)}+B_*u_t+w_t,
\end{align}
where $w_t\simiid N(0,I)$ and $u_t=\sum_{i=1}^{H}M^{[i-1]}w_{t-i}$  ($w_t=0$ for $t\leq0$). Observe that the coupling comes from the shared controls and disturbances. Let $z_t^{(1)}=\begin{pmatrix} x_t^{(1)}\\ u_t\end{pmatrix}$ and $z_t^{(2)}=\begin{pmatrix} x_t^{(2)}\\ u_t\end{pmatrix}$. We prove the following claim.
\begin{claim}\label{cl:id-approx}
The matrix $\whA$ is $(\kappa,\gamma/2)$-strongly stable \footnote{This means that there exists decomposition $\whA=Q\Lambda Q^{-1}$ with $\|\Lambda\|\leq 1-\gamma/2$, and $\|Q\|,\|Q^{-1}\|\leq \kappa$. }. Furthermore, for all $t\geq 2H+2$, we have
\begin{align}\label{eq:approx:auxiliary-diff}
 \Bigg|  \mathbb{E}_{w}\left\|x_{t}^{(1)}- x_{t}^{(2)}\right\|-\mathbb{E}_{\eta}\left \|x(M\ |\ \whA,\whB,\eta)-x(M\ |\ A_*,B_*,\eta)\right\| \Bigg| \leq 1/T,
\end{align}
where $w$ denotes the disturbance sequence $(w_t)_t$.
\end{claim}
\begin{proof}
 We use the assumption that
$\|\Delta\|\leq \frac{\gamma}{2\kappa^2}$, which implies that $ \|\whA-A_*\|\leq \frac{\gamma}{2\kappa^2}$. Also, from Assumption \ref{assump:stability}, we have $A_*=Q\Lambda Q^{-1}$ with $\|\Lambda\|\leq 1-\gamma$, and $\|Q\|,\|Q^{-1}\|\leq \kappa$. So, we get
\begin{align}
    \whA=Q\Lambda Q^{-1} +\whA-A_*= Q\left(\Lambda+ Q^{-1}\left(\whA-A_*\right)Q\right)Q^{-1}.
\end{align}
Also, $\left\|\Lambda+ Q^{-1}\left(\whA-A_*\right)Q \right\|\leq 1-\gamma+\kappa^2 \gamma/(2\kappa^2)=1-\gamma/2$. Thus, we proved that $\whA$ is $(\kappa,\gamma/2)$-strongly stable. Now, we prove the inequality \ref{eq:approx:auxiliary-diff}.
\begin{align}
    &\mathbb{E}_{w}\left\|x_{t}^{(1)}- x_{t}^{(2)}\right\| \notag \\ &=\mathbb{E}_{w}\left\|\whA^{H+1}x_{t-H-1}^{(1)}+\sum_{i=1}^{2H+1}\Psi_i(M\ |\ \whA,\whB)w_{t-i}- A_*^{H+1}x_{t-H-1}^{(2)}-\sum_{i=1}^{2H+1}\Psi_i(M\ |\ A_*,B_*)w_{t-i}\right\| \notag \\
   & \leq \left\| \whA^{H+1} \right \| \cdot \mathbb{E}_{w}\left\|x_{t-H-1}^{(1)}\right\|+\left\|A_*^{H+1}\right \|\cdot \mathbb{E}_{w}\left\|x_{t-H-1}^{(2)}\right\|  \notag \\ 
   &\ \ \ \ \ \ \ \ \ \ \ \ \  \ \ \ \ \ \ \ \ \ \  \ \ \ \ \ \ \ \ \ \ \ \ \ \ \ \ \ \ \ \  \ \ \ \ \ \  \ \   \ \ \ +\mathbb{E}_{w}\left\|\sum_{i=1}^{2H+1}\Psi_i(M\ |\ \whA,\whB)w_{t-i}-\sum_{i=1}^{2H+1}\Psi_i(M\ |\ A_*,B_*)w_{t-i}\right\|.
\end{align}
Since $t\geq 2H+2$, the third term is exactly $\mathbb{E}_{\eta}\left \|x(M\ |\ \whA,\whB,\eta)-x(M\ |\ A_*,B_*,\eta)\right\|$. Now, we show that the first two terms are small. Since $A_*$ is $(\kappa,\gamma)$-strongly stable and $\whA$ is $(\kappa,\gamma/2)$-strongly stable, we have $\left \|A_*^{H+1}\right\|\leq \kappa^2 (1-\gamma)^{H+1}$ and $\left \|\whA^{H+1}\right\|\leq \kappa^2 (1-\gamma/2)^{H+1}$, where $H$ is defined in \ref{subs:disturbance-policies}. Using again the strong stability of $\whA$ and $A_*$, we show in Appendix \ref{appdx:sec:aux-claims} (Claim \ref{cl:range-aux}), that $\mathbb{E}_{w}\left\|x_{t-H-1}^{(1)}\right\|,\mathbb{E}_{w}\left\|x_{t-H-1}^{(2)}\right\|\leq O(\kappa^2\beta\gamma^{-1}G)\cdot \sqrt{d_x}$. The way we chose $H$ finishes the proof.
\end{proof}

Now, we fix a $t\geq 2H+2$ , whose exact value we choose later. We will bound $ \mathbb{E}_{w}\left\|x_{t}^{(1)}- x_{t}^{(2)}\right\|$. First, we write a recursive formula for $x_{t}^{(1)}- x_{t}^{(2)}$:
\begin{align}\label{eq:rec-sketch}
     x_{t}^{(1)}- x_{t}^{(2)}&= \whA x_{t-1}^{(1)}+\whB u_{t-1}-A_*x_{t-1}^{(2)}-B_*u_{t-1} \notag \\
     &= (\whA -A_*)x_{t-1}^{(2)} +(\whB -B_*)u_{t-1}+\whA(x_{t-1}^{(1)}-x_{t-1}^{(2)}) \notag \\
     &= \Delta \cdot z_{t-1}^{(2)}+\whA(x_{t-1}^{(1)}-x_{t-1}^{(2)}).
\end{align}
By repeating \ref{eq:rec-sketch}, we get
\begin{align}\label{eq:rec}
     x_{t}^{(1)}- x_{t}^{(2)}&=\sum_{i=0}^{H-1}\whA^i \cdot \Delta \cdot z_{t-i-1}^{(2)}+\whA^{H}\left(x_{t-H-1}^{(1)}-x_{t-H-1}^{(2)}\right).
\end{align}
The previous bounds on $\left\|\whA^H\right\|,\mathbb{E}_{w}\left\|x_{t-H-1}^{(1)}\right\|,\mathbb{E}_{w}\left\|x_{t-H-1}^{(2)}\right\|$ imply that the second term is negligible, i.e., at most $1/T$. By applying triangle inequality, we get
\begin{align}
\mathbb{E}_{w} \|  x_{t}^{(1)}- x_{t}^{(2)}
    \|\leq \sum_{i=0}^{H-1}\left\|\whA^i\right\|& \cdot \mathbb{E}_{w}\left\| \Delta \cdot z_{t-i-1}^{(2)}\right\|+1/T.
\end{align}
Now, we prove a claim which shows that for large $t$, the term $\mathbb{E}_{w}\left\| \Delta \cdot z_{t-i-1}^{(2)}\right\|$ is essentially time-independent, for all $i\in \{0,1,\dots,H-1\}$.
\begin{claim}
For all $s\geq 2H+2$, we have 
\begin{align}
    \left|  \mathbb{E}_{w}\left\| \Delta \cdot z_{s}^{(2)}\right\|-  \mathbb{E}_{\eta}\left\| \Delta \cdot z(M\ |\ A_*,B_*,\eta) \right\| \right|\leq 1/T.
\end{align}
\end{claim}
\begin{proof}
It suffices to show that $\mathbb{E}_{w}\left\|\Delta\left(z_s^{(2)}-z\left(M\ |\ A_*,B_*,\eta(w)\right)\right) \right\|\leq 1/T$, where we define $\eta(w):=(w_{s-1},w_{s-2},\dots,w_{s-2H-1})$. Since $\|\Delta\|\leq \gamma/(2\kappa^2)\leq 1$, we have
\begin{align}
   & \mathbb{E}_{w}\left\|\Delta\left(z_s^{(2)}-z\left(M\ |\ A_*,B_*,\eta(w)\right)\right) \right\|\leq \mathbb{E}_{w}\left\|z_s^{(2)}-z\left(M\ |\ A_*,B_*,\eta(w)\right) \right\| \notag \\
    & \leq \mathbb{E}_{w}\left\|x_s^{(2)}-x\left(M\ |\ A_*,B_*,\eta(w)\right) \right\|+\mathbb{E}_{w}\left\|u_s-u\left(M\ |\ \eta(w)\right) \right\| \notag \\
    &\leq \mathbb{E}_{w}\left\|A_*^{H+1}x_{s-H-1} \right\|+0 \notag \\
    &\leq \left\|A_*^{H+1}\right \| \cdot \mathbb{E}_{w}\|x_{s-H-1}\|.
\end{align}
Again, our bound on $\left\|A_*^H\right\|$ and Claim \ref{cl:range-aux} in Appendix \ref{appdx:sec:aux-claims} finish the proof.
\end{proof}
\par
Now, we choose $t=3H+2$, which gives 
\begin{align}
  \mathbb{E}_{w}   \| \ x_{t}^{(1)}- x_{t}^{(2)}\|\leq \left(\mathbb{E}_{\eta}\left\| \Delta \cdot z(M\ |\ A_*,B_*,\eta) \right\|+1/T\right)\cdot   \sum_{i=0}^{H-1}\left\|\whA^i\right\| +1/T.
\end{align}
Also, we have $\mathbb{E}_{\eta}\left\| \Delta \cdot z(M\ |\ A_*,B_*,\eta) \right\|\leq \left(\mathbb{E}_{\eta}\left\| \Delta \cdot z(M\ |\ A_*,B_*,\eta) \right\|^2\right)^{1/2} = \left\|\Delta^T \right\|_{\Sigma(M)}$, and since $\whA$ is $(\kappa,\gamma/2)$-strongly stable, $\sum_{i=0}^{H-1}\left\|\whA^i\right\|\leq 4\kappa^2\gamma^{-1}$.
\par Finally, combining with Claim \ref{cl:id-approx}, we get  
\begin{align}
    \mathbb{E}_{\eta}\left \|x(M\ |\ \whA,\whB,\eta)-x(M\ |\ A_*,B_*,\eta)\right\|&\leq 4\kappa^2\gamma^{-1}\left(\left\|\Delta^T \right\|_{\Sigma(M)} +1/T\right)+2/T \notag\\
    &\leq 6\kappa^2\gamma^{-1}\left(\left\|\Delta^T \right\|_{\Sigma(M)} +1/T\right).
\end{align} 
\end{proof}
Now, we proceed with the rest of the proof. We define the average regret:
\begin{align}
    R_T^{avg}=\sum_{t=1}^T\mC\left(M_t\ | \ A_*,B_*\right)-T\cdot\mC\left(M_*\ 
     |\ 
      A_*,B_*\right ),
\end{align}
where $M_t$ is the policy executed at time $t$ and $M_*\in \argmin_{M\in \mM}\mC(M\ |\ A_*,B_*)$. As in Section \ref{sec:warmup-case}, we bound the difference $R_T-R_T^{avg}$. 
\begin{lemma}\label{lem-main:difference}
Let $C_4=\kappa^2\beta\gamma^{-3/2}G$ and $C_7=\kappa^4\beta^2\gamma^{-2}G$. With high probability, 
\begin{align}
    R_T-R_T^{avg}\leq \widetilde{O}\left(C_4 \right)\cdot (d_x+d_u)^{3/2}\sqrt{T}+\tO(C_7)\cdot d_x^{3/2}d_u.
\end{align} 
\end{lemma}
The proof of this lemma is more technical than the proof of the analogous Lemma, for two reasons. First, because of the dependencies between different timesteps. This can be alleviated with mixing-time arguments and the fact that the number of times the algorithm switches policy is only polylogarithmic in $T$. Second, because the algorithm uses $(\whw_t)_t$ instead of $(w_t)_t$. This is resolved using Lemma \ref{lem-main:noise-approx}. We provide the formal proof of Lemma \ref{lem-main:difference} in Appendix \ref{appdx:prf:difference}. We proceed with the analog of Lemma \ref{lem:cost-estimation-warmup}.

\begin{lemma}\label{lem-main:loss_approx}
 With high probability, for all epochs $r$ and for all $M\in \mM_r$, we have
\begin{align}\label{eq:loss_bound-main}
    \left|\mC(M\ |\ \whA_{t_r} , \whB_{t_r})-\mC(M\ |\ A_*,B_*)\right|\leq 2^{-r}.
\end{align}
\end{lemma}
\begin{proof-sketch}
We fix an $M\in \mM_r$. From Lemma \ref{lem:LS}, we have $\|\Delta_t\|\leq \tO(1)\cdot (d_x+d_u)/\sqrt{\lambda}\leq \gamma/(2\kappa^2).$, because of the way we chose $\lambda$ in Algorithm \ref{alg:LS}. Thus, from Lemma \ref{lem:reduction-cost-matrices}, 
\begin{align}
    \left|\mC(M\ | \ \whA_{t_r} , \whB_{t_r})-\mC(M\ |\ A_*,B_*)\right|\leq 6\kappa^2\gamma^{-1}\Big(\left\|\Delta_{t_r}^T \right\|_{\Sigma(M)}+1/T\Big).
\end{align}
Now, for for notational convenience, we define $M_{r,d+1}=M_{r,d+2}=\dots=M_{r,2d}=M_{r,0}$, and we show that since $\{M_{r,j}\}_{j=0}^d$ is an affine 2-barycentric spanner of $\mM_r$, we have 
\begin{align}
    \Sigma(M)\preccurlyeq 18d\sum_{j=1}^{2d}\Sigma(M_{r,j}),
\end{align}
which we prove in Appendix \ref{appdx:known} (Lemma \ref{lem-main:cov-barycentric}). This implies that $\| \Delta_{t_r}^T\|_{\Sigma(M)}^2\leq 18d \cdot  \sum_{j=1}^{2d}\| \Delta_{t_r}^T\|_{\Sigma(M_{r,j})}^2$. The final step is showing that with high probability,
\begin{align}\label{eq:exploratory-policies-bound}
    \sum_{j=1}^{2d}\| \Delta_{t_r}^T\|_{\Sigma(M_{r,j})}^2\leq  2^{-2r}\cdot \frac{\gamma^2}{ 12^2\cdot 18 \cdot d\kappa^4}\ .
\end{align}
The additional difficulties here (compared to the warmup case) are similar to the ones in proving Lemma \ref{lem-main:difference}, and we address them with the same techniques. We present the formal proof of \ref{eq:exploratory-policies-bound} in Appendix \ref{appdx:known} (Lemma \ref{lem-main:exploratory_policies}). Combining, we get $ \left|\mC(M\ | \ \whA_{t_r} , \whB_{t_r})-\mC(M\ |\ A_*,B_*)\right|\leq 2^{-r}/2+6\kappa\gamma^{-1}/T$, so the assumed lower bound on $T$ finishes the proof.
\end{proof-sketch}
Now, following almost the same steps as in the end of the proof in Section \ref{sec:warmup-case}, we can get the following lemma.

\begin{lemma}\label{lem-main:stat-regret-bound}
Let $C_5=\kappa^8\beta^2 \gamma^{-5}G$ and $C_8=\kappa^2\gamma^{-2}$. With high probability, 
\begin{align}
   R_T^{avg}\leq  \widetilde{O}\left(C_8\right)\cdot d_xd_u(d_x+d_u)\sqrt{T}+ \tO\left(C_5 \right)\cdot (d_x+d_u)^{6.5}
\end{align}
\end{lemma}
We present the proof in Appendix \ref{appdx:prf:lem-main:stat-regret-bound}. Combining with Lemma \ref{lem-main:difference}, we get the desired bound for $R_T$.

\section{General case for bandit feedback} \label{sec:bandit-case}
To tackle online control with bandit feedback, we use the stochastic bandit convex optimization (SBCO) algorithm of \cite{agarwal2011stochastic} as a black-box.  Before we present our algorithm and the formal theorem statement, we briefly present the SBCO setting. In SBCO, $\mX$ is a convex subset of $\mathbb{R}^d$ with diameter bounded by $D$, and $f : \mX \rightarrow \mathbb{R}$ is an $L$-Lipschitz
convex function on $\mX$. The algorithm has access to $f$ via a noisy value oracle, i.e., it can query the value of any $x\in \mX$, and the response is
$y = f(x) + \zeta$
where $\zeta$ is an independent $\sigma^2$-subgaussian random variable with mean zero. The goal is to minimize regret: after making $n$ queries $x_1,\dots,x_n \in \mX$, the regret is $\sum_{t=1}^nf(x_t)-nf(x_*)$, where $x_*$ is a minimizer of $f$ over $\mX$. In \cite{agarwal2011stochastic}, the authors give a polynomial-time algorithm that takes as input $d,D,L,\sigma^2,n$ and a separation oracle for $\mX$, and achieves regret $\widetilde{O}(1)\cdot poly(d,\log{D},L,\sigma^2)\cdot\sqrt{n}$. 
\\
\\
The function $f$ will be the $\mC(\cdot \ |\ A_*,B_*)$. We now give the intuition behind our algorithm. Suppose we knew $(w_t)_t$, so that we can exactly execute some policy $M$. Suppose that we execute $M$ during the interval $[t-2H+1,t]$. Then, $x_t\approx \sum_{i=0}^{2H}\Psi_i(M\ |\ A_*,B_*)w_{t-i-1}$, and $u_t=\sum_{i=1}^H M^{[i-1]}w_{t-i}$. Thus, we have $\mathbb{E}[c(x_t,u_t)]\approx \mC(M \ |\ A_*,B_*)$, and $c(x_t,u_t)$ is independent of all $w_{t'}$, for $t'<t-2H-1$. So, the natural algorithm is to execute some policy $M$ for $2H+1$ steps, then send the last cost to the SBCO algorithm, which will decide the next policy $M'$, that will be the next one to execute. The only problem with this idea is that we do not know $(w_t)_t$. But, as we showed in Section \ref{sec:known-cost}, we can compute online estimates $(\whw_t)_t$, such that the average squared error is small (Lemma \ref{lem-main:noise-approx}). Given this lemma, the key step in the analysis of Algorithm \ref{alg:bandit-feedback} is proving that the SBCO algorithm of \cite{agarwal2011stochastic} is robust to adversarial noise in the responses, when this noise is small on average. The errors $\|\whw_t-w_t\|$ will play the role of this adversarial noise.
\SetAlgoNoLine

\SetKwInput{KwInput}{Input}              
\SetKwInput{KwOutput}{Output}
\begin{algorithm}[H]
  \textbf{Input}: SBCO algorithm with parameters $d,D,L,\sigma^2,n$, domain $\mX=\mM$. Initial estimates $A_0 , B_0$ satisfying $\|(A_0 \ B_0)-(A_* \ B_*)\|_F\leq \epsilon$.\\
  Set estimates of matrices $(\wh{A}_1\ \wh{B}_1) =(A_0\ B_0)$, disturbance estimates $\widehat{w}_t=0$ for $t\leq 0$.\\
  SBCO algorithm queries the first point $M$.\\
  Set initial policy $M_1=M$.\\
 \For{$t=1,\dots, T$}{
  Apply control $u_t=\sum_{i=1}^H M_t^{[i-1]}\wh{w}_{t-i}$.\\
  Observe $x_{t+1}$ and $c(x_t,u_t)$.\\
  Call System-Estimation (Algorithm \ref{alg:LS}), to get $\whA_{t+1},\whB_{t+1}$.\\
  Record the estimate $\widehat{w}_t=x_{t+1}-\widehat{A}_{t+1} x_t -\widehat{B}_{t+1} u_{t}$.\\
  \eIf{\ $t\mod (2H+1)=0$}{
   Send $c(x_t,u_t)$ to SBCO algorithm.\\
   SBCO algorithm queries a new point $M$.\\
   Set $M_{t+1}=M$.
   }{ 
    Set $M_{t+1}=M_t$.
  }
 }
\caption{Control of unknown LDS with bandit feedback}
\label{alg:bandit-feedback}
\end{algorithm}

 \par  We now formally state our theorem, which says that after appropriately initializing the input parameters of the SBCO algorithm,  Algorithm \ref{alg:bandit-feedback} achieves $\sqrt{T}$-regret. 
\begin{theorem}\label{thm-main:bandit}
There exist $C_1,C_2,C_3,C_4,C_5= poly\left(d_x,d_u,\kappa,\beta,\gamma^{-1},G,\log{T}\right)$, such that after initializing the SBCO algorithm with $d=d_x\cdot d_u\cdot H$, $D=C_1$, $L=C_2$, $\sigma^2=C_3$ and $n=T/(2H+2)$, if the horizon $T\geq C_4$, then with high probability, warmup exploration (Algorithm \ref{alg:warm-up} in Appendix \ref{appdx_subs:warm-up}) followed by Algorithm \ref{alg:bandit-feedback} satisfy $R_T\leq C_5 \cdot \sqrt{T}.$
\end{theorem}\label{thm:bandit}
Our message here is that $\sqrt{T}$-regret is achievable in polynomial time, so we did not try to optimize the terms $C_i$.
The proof is in Appendix \ref{appdx:bandit}. 

\section{Extensions}\label{sec:extensions}
\paragraph{General stochastic disturbances:}Other than Gaussian, we can deal with any stochastic bounded disturbance distribution. The only place where the assumption of Gaussian disturbances really helps is that given some policy $M$
and matrices $\whA, \whB$, we can compute offline (to a very good approximation) the stationary cost $C(M\ |\ \whA,\whB)$,
because we know the disturbance distribution. However, even when we do not, we can still use the estimated
disturbances $(\hat{w}_t)_t$ as samples to approximate this expectation (i.e., the average cost).

\paragraph{Partial observation:}The extension to partial observation is tedious but straightforward and uses the idea of “nature’s y’s”,
exactly as in \cite{simchowitz2020improper}.
\section{Summary and open questions}
We gave the first polynomial-time algorithms with optimal regret, with respect to the time horizon, for online control of LDS with general convex costs and comparator class the set of DFCs. Our main result was a novel geometric exploration scheme for the case where the cost function is known.
The following open questions arise. First, can we improve the $\tO(C)\cdot d_xd_u(d_x+d_u)\sqrt{T}$ regret bound, in terms of dimension dependence? This looks plausible because the barycentric spanners are constructed by treating the policies as flattened vectors of dimension $d_xd_u H$, thus the matrix structure is not exploited. Second, Algorithm \ref{LJ-exploration} is not practical, since it employs the ellipsoid method. Is there a simpler, gradient-based algorithm that also achieves $\sqrt{T}$-regret? Third, a challenging question is whether $\sqrt{T}$-regret is achievable for nonstochastic control, where the disturbances are adversarial and the cost function adersarially changes over time. Even more broadly, can we prove regret bounds with respect to interesting nonlinear, yet tractable policy classes?

\bibliography{ref}
\bibliographystyle{plain}

\newpage
\appendix

\section*{Additional notation}
For two matrices $A,B$ we write $\langle A, B\rangle$ to denote the matrix inner product $tr(A^TB)$. We write $|A|$ to denote the determinant $det(A)$. Finally, for $x,y\in \mathbb{R}$ which depend on the problem parameters we write $x\lesssim y$ to denote that $x\leq O(1)\cdot y$.
\section{Initial stabilizing policy}\label{appdx:sec:init-stable-policy}
Linear policies are parameterized with a matrix K and apply controls $u_t=Kx_t$. We define the class of strongly-stable linear policies. This definition was introduced in \cite{cohen2018online} and quantifies the classical notion of a stable policy.
\begin{definition}
A linear policy $K$ is $(\kappa,\gamma)$-strongly-stable if there exists a decomposition of $A_*+B_*K=Q\Lambda Q^{-1}$, where $\|\Lambda\|\leq 1-\gamma$ and $\|K\|,\|Q\|,\|Q^{-1}\|\leq \kappa$.
\end{definition}
To deal with unstable systems, we assume that the learner is initially given a strongly-stable linear policy.
\begin{assumption}\label{assump:init-stable-policy}
The learner is initially given a $(\kappa,\gamma)$-strongly-stable policy $K_0$, for some known $\kappa\geq 1$ and $\gamma>0$.
\end{assumption}
Let $\mathcal{K}$ be the set of all $(\kappa,\gamma)$-strongly-stable linear policies. For unstable systems, we set the policy class $\Pi$ in \ref{eq:regret_intro} to be the more general class of DFCs, which apply controls $u_t=Kx_t+\sum_{i=1}^HM^{[i-1]}w_{t-i}$, where $K\in \mathcal{K}$.
In other words, the policies have one extra component which serves for stabilizing the system. So, let 
\begin{align}
\mM^{unst}=\left\{(K,M)\ \Big{|} \ M=\left(M^{[0]},\dots,M^{[H-1]}\right),\  \sum_{i=1}^H \left\|M^{[i-1]}\right\|\leq G^{unst}, \  K\in \mathcal{K}\right\}.
\end{align}
We can replace Assumption \ref{assump:stability} with Assumption \ref{assump:init-stable-policy} and get almost the same regret bounds (with respect to $\mM^{unst}$), by making two small changes in our algorithms. First, we set $G=G_{unst}+\kappa^3/\gamma$. Second,  instead of playing the control $u_t$, suggested by Algorithms \ref{LJ-exploration} and \ref{alg:bandit-feedback}, we play the control $\widetilde{u}_t=K_0 x_t+ u_t$. We now explain why these changes allow us to use off-the-shelf our regret bounds for $\mM$, to get the same bounds for $\mM^{unst}$.
\par Let $J_{A,B}(\pi)$ be the average infinite horizon cost of a policy $\pi$, with respect to the system matrices $A,B$. Clearly, it suffices to show that for all $(K,M)\in \mM^{unst}$, there exists $\widetilde{M}\in \mM$, such that  $J_{A_*,B_*}(K,M)+O(1/T)\geq J_{A_*+B_*K_0,B_*}\left(\widetilde{M}\right)$. We will prove this inequality. Note that under policy $(K,M)$, we have $u_t=K_0x_t+(K-K_0)x_t+\sum_{i=1}^HM^{[i-1]}w_{t-i}$. So, executing it for the system $(A_*,B_*)$ is the same as executing the policy $u_t=(K-K_0)x_t+\sum_{i=1}^HM^{[i-1]}w_{t-i}$ for the system $(A_*+B_*K_0,B_*)$. Now, we will show that there exist $\wt{M}_{1},\wt{M}_2\in \mM$, such that 

$$u_t^{\wt{M}_{1}}\approx u_t^{M}, \  u_t^{\wt{M}_{2}}\approx u_t^{K-K_0},\ \text{and}\ \ \wt{M}_{1}+\wt{M}_{2}\in \mM.$$ For the first condition, we simply choose $ \wt{M}_{1}=M$. For the second, we will need the following claim, which we prove at the end of this section.
\begin{claim}
Let $(A,B)$ be the matrices of an LDS. Also, let $A$ be  $(\kappa,\gamma)$-strongly stable, and $K_1$ be a $(\kappa,\gamma)$-strongly stable linear policy. Then, there exists $\wt{M}$, such that $\sum_{i=0}^{H-1}\left\|M^{[i]}\right\|\leq \kappa^3/\gamma$, and $\mathbb{E}_{w}\left\|u_t^{\wt{M}}-u_t^{K_1}\right\|\leq 1/P$, where $P$ is a large polynomial in $T,d_x,d_u,\kappa,\beta,\gamma^{-1},G^{unst}$, and $w$ denotes the sequence $(w_t)_t$.
\end{claim}
We use the claim for $A=A_*+B_*K_0$, $B=B_*$, $K_1=K-K_0$ and $\wt{M}=\wt{M}_2$. We get that $\mathbb{E}_w\left\|u_t^{\wt{M}_2}-u_t^{K-K_0}\right\|\leq 1/P$, which implies that $\mathbb{E}_w\left\|u_t^{\wt{M}_1}+u_t^{\wt{M}_2}-\left(u_t^{K-K_0}+u_t^{M}\right)\right\|\leq 1/P$. Clearly, $u_t^{\wt{M}_1+\wt{M}_2}=u_t^{\wt{M}_1}+u_t^{\wt{M}_2}$.
Also, it can be easily verified that $u_t^{(K-K_0,M)}=u_t^{\wt{M}}+u_t^{K-K_0}$. Thus, \begin{align}\label{eq:u_are_close}
    \mathbb{E}_w\left\|u_t^{\wt{M}_1+\wt{M}_2}-u_t^{(K-K_0,M)}\right\|\leq 1/P.
\end{align}  
Notice that $\wt{M}_1+\wt{M}_2\in \mM$, since we chose $G=G^{unst}+\kappa^3/\gamma$. Now, since the system matrices are $(A_*+B_*K_0,B_*)$, for all $t$, we have
\begin{align}
   \mathbb{E}_w \left\|x_{t+1}^{\wt{M}_1+\wt{M}_2}-x_{t+1}^{(K-K_0,M)}\right\|& \leq \sum_{i=0}^{t-1}\| A_*+B_*K_0^i\| \cdot \|B_*\|\cdot \mathbb{E}_w\left\|u_{t-i}^{\wt{M}_1+\wt{M}_2}-u_{t-i}^{(K-K_0,M)}\right\| \notag \\
    & \leq 1/P \cdot \sum_{i=0}^{t-1}\| A_*+B_*K_0^i\| \cdot \|B_*\|\leq \kappa^2\beta \gamma^{-1}/P,
\end{align}
where we used that $A_*+B_*K_0$ is a $(\kappa,\gamma)$-strongly stable matrix (Assumption \ref{assump:init-stable-policy}). Combining with \ref{eq:u_are_close} and the fact that $c$ is 1-Lipschitz, we get that 
$$ J_{A_*+B_*K_0,B_*}\left(\wt{M}_1+\wt{M}_2\right)\leq J_{A_*+B_*K_0,B_*}(K-K_0,M)+O(1/T)=J_{A_*,B_*}(K,M)+O(1/T).$$ 
We now prove the claim.
\begin{proof}
We have $x_{t+1}^{K_1}=\sum_{i=0}^{t-1}(A+BK_1)^iw_{t-i}$. Thus,
\begin{align}
    u_{t+1}^{K_1}=K_1\sum_{i=0}^{t-1}(A+BK_1)^iw_{t-i}.
\end{align}
Let $\wt{M}^{[i]}=K_1(A+BK_1)^i$, for $i=0,1,\dots, H-1$.  We have

\begin{align}
    \mathbb{E}_w\|u_{t+1}^{K_1}-u_{t+1}^{\wt{M}}\|& \leq \mathbb{E}_w\|K_1\sum_{i=H}^{t-1}(A+BK_1)^iw_{t-i}\|\leq \sqrt{d_x} \sum_{i=H}^{t-1}\|K_1\| \cdot \|(A+BK_1)^i\| \notag \\
    &\leq \kappa\sqrt{d_x} \sum_{i=H}^{t-1}\kappa^2(1-\gamma)^i \leq \kappa^3\sqrt{d_x}(1-\gamma)^H/\gamma\leq 1/P,
\end{align}
where we used that $K_1$ is a $(\kappa,\gamma)$-strongly stable policy, with respect to the system $(A,B)$.
\end{proof}

\section{Least-Squares: Proof of Lemma \ref{lem:LS}}\label{appdx_subs:LS}
We first state Lemma 6 of \cite{cohen2019learning}.
\begin{lemma}\label{lem:LS_cohen}
 Let $V_0=\lambda \cdot I$. With high probability, for all $t$,
\begin{align}\label{eq:ls_det}
    \| \Delta_t^T\|_{V_t}^2 \leq \tO(1)\cdot  (d_x+d_u)\log{\frac{|V_t|}{|V_0|}}+\lambda\|\Delta_0\|_F^2.
\end{align}
\end{lemma}

Since $|V_0|=\lambda^{d_x+d_u}$, $\|\Delta_0\|_F\leq \epsilon$ and $\lambda\epsilon^2\leq \widetilde{O}(1)\cdot(d_x+d_u)^2$ (condition of Theorem \ref{thm:LJ-exploration}), we get that with high probability, for all $t$,
\begin{align}\label{eq:LS_cohen}
 \| \Delta_t^T\|_{V_t}^2 \leq \tO(1)\cdot  (d_x+d_u)^2 \left(1+\log{\left|V_t\right|^{(d_x+d_u)^{-1}}}\right).
\end{align}
We bound $|V_t|$ with the following claim.

\begin{claim}\label{cl:AMGM}
For all $t$,
\begin{align}
    \left|V_t\right|^{(d_x+d_u)^{-1}}\leq \frac{1}{d_x+d_u} \cdot \sum_{s=1}^{t-1}\|z_s\|^2 +\lambda(d_x+d_u).
\end{align}
\end{claim}
\begin{proof}
From AM-GM inequality, 
\begin{align}
    \left|V_t\right|^{(d_x+d_u)^{-1}}\leq  \frac{1}{d_x+d_u} \cdot \trace(V_t)& = \frac{1}{d_x+d_u} \cdot \trace \left(\sum_{s=1}^{t-1}z_sz_s^T+\lambda I\right)  \notag \\
    &= \frac{1}{d_x+d_u} \cdot \sum_{s=1}^{t-1}\|z_s\|^2 +\lambda. \notag 
\end{align}
\end{proof}
So, it remains to control the magnitude of all $z_s$, for $s<t$.
\begin{lemma}\label{lem:range-z_t}
With high probability, Algorithm \ref{LJ-exploration} satisfies for all $t$,
 \begin{align}\label{eq:state-bounded}
     \|z_t\|\leq \widetilde{O}(\kappa^2 \beta \gamma^{-1}G) \cdot \sqrt{d_x}.
 \end{align}
\end{lemma}
\begin{proof}
From Claim \ref{cl:noise_bounded}, we have that with high probability, for all $t$, $\|w_t\|\leq \tO(\sqrt{d_x})$. So, it suffices to show that if both this bound on the disturbances and inequality \ref{eq:LS_cohen} hold for all $t$, then the bound in \ref{eq:state-bounded} holds. We prove this by induction on $t$. For $t=1$, we have $u_1=0$ and $\|x_1\|$ is bounded by assumption (see Theorem \ref{thm:LJ-exploration} statement). Suppose that $z_s$ satisfies \ref{eq:state-bounded} for all $s<t$. Then, Claim \ref{cl:AMGM} implies that for all $s<t$,
 \begin{align}
     \log{\left|V_{s+1}\right|^{(d_x+d_u)^{-1}}}\leq \tO(1).
 \end{align}
So, by inequality \ref{eq:LS_cohen}, we have
\begin{align}\label{eq:ind_hypoth}
    \left \| \Delta_{s+1}^T\right\|_{V_{s+1}}^2 \leq \tO(1)\cdot  (d_x+d_u)^2.
\end{align}
We now bound $u_{s}$, for $s\leq t$. Fix such an $s$. We have
 \begin{align}
     u_{s}=\sum_{i=1}^HM_s^{[i-1]}\widehat{w}_{s-i}=\sum_{i=1}^HM_s^{[i-1]}w_{s-i}+\sum_{i=1}^HM_s^{[i-1]}\left(\widehat{w}_{s-i}-w_{s-i}\right).
 \end{align}
 We show the following claim.
 \begin{claim}\label{cl:instant-noise-approx}
 For all $\tau < t$, we have $\|\widehat{w}_{\tau}-w_\tau\|\leq\tO(1)$.
 \end{claim}
 \begin{proof}
 We have 
 \begin{align}
     \|\widehat{w}_{\tau}-w_\tau\|^2= \|\Delta_{\tau+1}\cdot z_{\tau}\|^2\leq \widetilde{O}(\kappa^{4} \beta^2 \gamma^{-2}G^2) \cdot d_x \cdot \|\Delta_{\tau+1}\|^2.
 \end{align}
 From inequality \ref{eq:ind_hypoth} and the fact that $\lambda\cdot I\preccurlyeq V_{s+1}$, we get that $\|\Delta_{\tau+1}\|^2\leq \tO(1)\cdot (d_x+d_u)^2/\lambda$. The claim follows from our choice of $\lambda$.
 \end{proof}
 Back to our fixed $s\leq t$, using Claim \ref{cl:instant-noise-approx} and the disturbance bound, we get
  \begin{align}
    \| u_{s}\|\leq \tO(1)\cdot  \sqrt{d_x}\cdot \sum_{i=1}^H \|M_s^{[i-1]}\|\leq  \tO(G)\cdot \sqrt{d_x}\ .
 \end{align}
 We will now bound $\|x_t\|$.
   \begin{align}
     \|x_{t}\|&=\left\|A_{*}^{t-1}x_1 +\sum_{i=1}^{t-1}A_{*}^{i-1}\left(w_{t-i}+B_{*}u_{t-i}\right) \right\| \notag \\
     &\leq \|A_{*}^{t-1}\|\|x_1\| +\sum_{i=1}^{t-1}\|A_{*}^{i-1} \| \left(\|w_{t-i}\|+\|B_{*}\|\|u_{t-i}\|\right) \notag \\
     & \leq \tO(\beta\kappa^2\gamma^{-1/2})\cdot \sqrt{d_x}\cdot \|A_{*}^{t-1}\| +\tO(\beta G)\sqrt{d_x} \sum_{i=0}^{\infty}\|A_{*}^{i} \| \notag \\
     &\leq \tO(\kappa^2\beta\gamma^{-1}G)\cdot \sqrt{d_x},\notag
 \end{align}
where the last step follows from $(\kappa,\gamma)$-strong stability of $A_*$ (see Claim \ref{cl:powers_A}). The fact that $\|z_t\|\leq \|x_t\|+\|u_t\|$ finishes the proof.
\end{proof}

We can now finish the proof of Lemma \ref{lem:LS}. With high probability, both \ref{eq:LS_cohen} and \ref{eq:state-bounded} hold for all $t$, so $\log{|\overline{V}_t|^{(d_x+d_u)^{-1}}}\leq \tO(1)$, which after being plugged-in inequality \ref{eq:LS_cohen} completes the proof.
\section{Proofs for known cost function}\label{appdx:known}

\subsection{Proof of Lemma \ref{lem-main:difference}}\label{appdx:prf:difference}
In the proof, we use $\mC(M)$ to refer to $\mC(M\ |\ A_*,B_*)$.
\begin{align}
  R_T-R_T^{avg}& = \sum_{t=1}^T \left(c(z_t)-\mC(M_t)\right)+T\min_{M\in\mM} \mC(M)-T\min_{M\in\mathcal{M}} J(M) \\
  &\leq \sum_{t=1}^T \left(c(z_t)-\mC(M_t)\right) +1,
\end{align}
where we used Theorem \ref{thm:truncation}. We will need some notations and  definitions. As we mentioned in the main text, for notational convenience, we define $M_{r,j}=M_{r,0}$, for $j=d+1,d+2,\dots,2d$. In the proofs, we substitute the \say{Call Execute-Policy($M_{r,0},d\cdot T_r$)} with \say{for $j=d+1,d+2,\dots 2d$ do Execute-Policy($M_{r,j}, T_r$) end.} Clearly, these are equivalent, but the latter will lead to simpler formulas. Now, for $j=1,2,\dots, 2d$, we define $I_{r,j}\subseteq [T]$ to be the interval of execution of $M_{r,j}$ and $t_{r,j}\in [T]$ to be the first step of this interval. Let $H':=2H+1$. For all $h=0,1,\dots,H'$, we define
\begin{align}
    I_{r,j,h}=\{t\in I_{r,j}\ |\ t=t_{r,j}+H' \cdot k +h, \ k\geq 1\}.
\end{align} 
Also, let $I_{r,j}'= \{t_{r,j},t_{r,j}+1,\dots,t_{r,j}+H'-1\}$, i.e.,  the first $H'$ steps of $I_{r,j}$. Observe that $\cup_{h=0}^{H'-1}I_{r,j,h}=I_{r,j}\setminus I_{r,j}'$. Let $q$ be the total number of epochs.
\begin{align}
    &\sum_{t=1}^T \left(c(z_t)-\mC(M_t)\right)=\sum_{r=1}^q\sum_{j=1}^{2d}\sum_{t\in I_{r,j}}\left(c(z_t)-\mC(M_t)\right) \notag \\
    &=\sum_{r=1}^q\sum_{j=1}^{2d}\sum_{h=0}^{H'-1}\sum_{t\in I_{r,j,h}}\left(c(z_t)-\mC(M_t)\right)+\sum_{r=1}^q\sum_{j=1}^{2d}\sum_{t\in I_{r,j}'}\left(c(z_t)-\mC(M_t)\right)
\end{align}

We call the first term $S_1$ and the second $S_2$, and we bound them via the following two claims.

\begin{claim}\label{appdx:cl:martingale-analysis}
With high probability, $
    S_1\leq \tO(\kappa^2\beta\gamma^{-3/2}G)\cdot\sqrt{(d_x+d_u)^3T}.$
\end{claim}

\begin{claim}\label{cl:few-switches}
With high probability, $
   S_2\leq \widetilde{O}(\kappa^4\beta^2\gamma^{-2}G)\cdot d_x^{3/2}d_u.$
\end{claim}
These two claims conclude the proof of Lemma \ref{lem-main:difference}. 
We first prove Claim \ref{cl:few-switches}. 
\begin{proof}
We will use the fact that the number of policy switches is small, i.e. logarithmic in $T$. First, we will need the following claim, which we prove in Appendix \ref{appdx:sec:aux-claims} (Claim \ref{appdx:cl:aux-sequence}).

\begin{claim}\label{cl:instant-regret}
With high probability, for all $t$,
\begin{align}
    c(z_t)-\mC(M_t)\leq \widetilde{O}(\kappa^4\beta^2\gamma^{-1}G) \cdot \sqrt{d_x}.
\end{align}
\end{claim}

Using Claim \ref{cl:instant-regret}, $S_2\leq q\cdot2d\cdot H'\cdot  \widetilde{O}(\kappa^4\beta^2\gamma^{-1}G)\cdot \sqrt{d_x} \notag
   \leq \widetilde{O}(\kappa^4\beta^2\gamma^{-2}G)\cdot d_x^{3/2}d_u.$

\end{proof}
Now, we prove Claim \ref{appdx:cl:martingale-analysis}.
\begin{proof}
We break $S_1$ into two terms: $S_3$ and $S_4$. We will bound $S_3$ via martingale concentration and $S_4$ will be errors coming from truncation-of- horizon-type arguments and the fact that the algorithm uses $\widehat{w}_t$ instead of $w_t$.  We consider an auxiliary state/control sequence $(\ox_t,\ou_t)_{t\in[T]}$, defined as
\begin{align}
    \overline{x}_t=\sum_{i=0}^HA_*^iw_{t-i-1}+\sum_{i=0}^HA_*^iB_*\overline{u}_{t-i-1},
\end{align}
where $\overline{u}_s=\sum_{i=1}^H M_s^{[i-1]} w_{s-i}$.\footnote{We set $w_t=0$, for all $t\leq 0$.} The differences with the actual sequence are 1) we truncated the time-horizon and 2) here the controls use the true disturbances. We also define $\overline{z}_t=\begin{pmatrix} \ox_t\\ \ou_t\end{pmatrix}$. From Appendix \ref{appdx:sec:aux-claims} (Claim \ref{appdx:cl:aux-sequence}), we have that with high probability, for all $t$,
\begin{align}\label{eq:approx-aux-sequence}
        \|\overline{z}_t-z_t\|\leq \tO(\kappa^2\beta\gamma^{-1}G)\cdot \sum_{i=1}^{2H+1}\|\whw_{t-i}-w_{t-i}\|+1/T.
\end{align}
So, we write
\begin{align}\label{eq:surrogate-error}
 S_1= \sum_{h=0}^{H'-1}\sum_{r=1}^q\sum_{j=1}^{2d}\sum_{t\in I_{r,j,h}}\left(c(\oz_t)-\mC(M_{r,j})\right)+  \sum_{r=1}^q\sum_{j=1}^{2d}\sum_{h=0}^{H'-1}\sum_{t\in I_{r,j,h}}\left(c(z_t)-c(\oz_t)\right)
\end{align}
We call the first sum $S_3$ and the second $S_4$. For $S_4$, with high probability, 
\begin{align}\label{eq:martingale-auxiliary}
   S_4\leq \sum_{t=1}^T \|z_t-\oz_t\|&\leq \tO(\kappa^2\beta\gamma^{-1}G)\sum_{t=1}^T \sum_{i=1}^{2H+1}\|\whw_{t-i}-w_{t-i}\|+1 \notag\\
    &\leq 
    \tO(\kappa^2\beta\gamma^{-1}G) \sum_{t=1}^T \|\whw_{t}-w_{t}\|+ 1,
\end{align}
where we used inequality \ref{eq:approx-aux-sequence} and the fact that $c$ is 1-Lipschitz. We now apply Lemma \ref{lem-main:noise-approx}, followed by Cauchy-Schwartz, to get $\sum_{t=1}^T \|\whw_{t}-w_{t}\|\leq \widetilde{O}(1)\cdot \sqrt{(d_x+d_u)^3T}$. Thus, $S_4\leq \tO(\kappa^2\beta\gamma^{-1}G)\cdot \sqrt{(d_x+d_u)^3T} $.
\par For $S_3$, we define $S_{3,h}:=\sum_{r=1}^q\sum_{j=1}^{2d}\sum_{t\in I_{r,j,h}}\left(c(\oz_t)-\mC(M_{r,j})\right)$, and we bound each of these separately. We consider the $\sigma$-algebra $\mathcal{F}_t=\sigma(w_1,w_2,\dots, w_{t-H'-1})$. We also fix a tuple $(h,r,j,t)$, where $h\in \{0,1,\dots,H'\}$, $r\in \{1,2,\dots,q\}$, $j\in\{1,2,\dots,2d\}$, and $t\in I_{r,j,h}$. Now, we fix an $h$ and we focus of $S_{3,h}$. Observe that for all $s<t$, if $c(\oz_s)-\mC(M_{s})$ participates in $S_{3,h}$, then it is $\mathcal{F}_t$-measurable.
 This is because of 1) the way the algorithm decides which policies to execute and 2) the definition of the sequence $(\oz_t)_t$. Moreover, the policy $M_{r,j}$ is also $\mathcal{F}_t$-measurable, because at time $t$ we have already spent at least $H'$ timesteps in epoch $r$, so everything that happened up until the end of epoch $r-1$ is $\mathcal{F}_t$-measurable, and so the same is true for $M_{r,j}$. Combining these observations with the definitions of $\oz_t$ and $\mC(M_{r,j})$, we get that $\mathbb{E}[c(\oz_t)-\mC(M_{r,j})|\mathcal{F}_t]=0$. To apply martingale concentration we will need the following claim. 
\begin{claim}\label{cl:gauss-concentr}
Let $L=\kappa^2\gamma^{-1}\beta G$. Then,  $\oz_t$ is $L
$-Lipschitz as a function of $(w_{t-H'},\dots,w_{t-1})$. Furthermore, conditioned on $\mathcal{F}_t$, the random variable $c(\oz_t)-\mC(M_{r,j})$ is $L^2$-subgaussian.
\end{claim}
\begin{proof}
 Since $M_{r,j}\in \mM$, for all $s\in \{t-H-1,t-H,\dots,t\}$, $\overline{u}_s$ is $L_u=\sum_{i=1}^H\|M_{r,j}^{[i-1]}\|\leq G$-Lipschitz. Thus, the Lipschitz constant of $\overline{x}_t$ is upper-bounded by 
\begin{align}
\sum_{i=0}^H\|A_*^i\|+\sum_{i=0}^H\|A_*^i\|\cdot\|B_*\|\cdot L_u\leq \kappa^2\gamma^{-1}+\kappa^2\gamma^{-1}\beta \cdot L_u, 
\end{align}
where we used Claim \ref{cl:powers_A}. Substituting $L_u$ finishes the proof of the first part. The second part follows from the fact that $c$ is 1-Lipschitz and from Gaussian concentration \cite{vershynin2018high}.
\end{proof}
From Azuma's inequality \cite{van2014probability}, for our fixed $h$, $S_{3,hj}$ is $L^2\frac{T}{3H}$-subgaussian, with $L$ as in Claim \ref{cl:gauss-concentr}. So with high probability, $S_{3,h}\leq O\left(\sqrt{\sigma^2T /H}\right)$. By applying union bound, we get that $S_3 \leq O\left(\sqrt{L^2T H}\right)=\widetilde{O}(\kappa^2\gamma^{-3/2}\beta G)\cdot \sqrt{T}$. 
\end{proof}
\qed

\begin{lemma}\label{lem-main:cov-barycentric}
For all $M\in \mM_r$, $ \Sigma(M)\preccurlyeq 18d\cdot \sum_{j=1}^{2d} \Sigma(M_{r,j}).$
\end{lemma}
\begin{proof}
We fix an $M\in \mM_r$. Since $\{M_{r,0},\dots,M_{r,d}\}$ is an affine 2-barycentric spanner of $\mM_r$, we can write $M=M_{r,0}+\sum_{j=0}^d\lambda_j\cdot (M_{r,j}-M_{r,0})$, where $\lambda_j\in[-2,2]$. Since we defined $M_{r,j}=M_{r,0}$ for all $j\in \{d+1,\dots,2d\}$, we can write $M=\sum_{j=1}^{2d}\lambda_j\cdot M_{r,j}$, where $\lambda_{j}=-\lambda_{j-d}+1/d$ for all $j\geq d+1$ (the other $\lambda_j$ stay the same). Thus, we have that all $\lambda_j\in[-3,3]$ and $\sum_{j=1}^{2d}\lambda_j=1$, i.e, it is an affine combination.
The next claim, takes us from policies to covariances.

\begin{claim}\label{cl-main:affine_transform}
The exists an affine transformation $T$, such that for all $M\in \mM$, $\Sigma(M)=T(M)T(M)^T$.
\end{claim}
\begin{proof}
$\Sigma(M)=\mathbb{E}_{\eta}\left[z(M\ |\ A_*,B_*,\eta)\cdot z(M\ |\ A_*,B_*,\eta)^T\right]$. For this proof, we write $\Psi_i(M)$, to denote $\Psi_i(M\ | \ A_*,B_*)$. We define
\begin{align}
T(M):=
\begin{pmatrix}
  \begin{matrix}
 \Psi_0(M) & \Psi_1(M)&\dots &\Psi_{H-1}(M) &\Psi_{H}(M) &\dots &\Psi_{2H}(M)\\
  M^{[0]} & M^{[1]}& & M^{[H-1]}&0 &  & 0
  \end{matrix}
\end{pmatrix}
\end{align}
and observe that $z(M\ |\ A_*,B_*,\eta)=T(M)\cdot \eta$. The transformation $T(\cdot)$ is affine due to the definition of $\Psi_i(M)$ (see Subsection \ref{subs:disturbance-policies}). The claim follows from the fact that $\mathbb{E}_{\eta}\left[\eta \eta^T\right]=I$.
\end{proof}
Back to our fixed $M\in\mM_r$, Claim \ref{cl-main:affine_transform} implies that $T(M)=\sum_{j=1}^{2d} \lambda_j T(M_{r,j})$. We have
\begin{align}
    \Sigma(M)=&T(M)T(M)^T =\left(\sum_{j=1}^{2d} \lambda_j T(M_{r,j}) \right)\left(\sum_{j=1}^{2d} \lambda_j T(M_{r,j}) \right)^T  \notag \\
    &\preccurlyeq \left(\sum_{j=1}^{2d} \lambda_j^2 \right)\left( \sum_{j=1}^{2d}T(M_{r,j})T(M_{r,j})^T \right)\preccurlyeq 18d\cdot \sum_{j=1}^{2d} \Sigma(M_{r,j}),
\end{align}
where we used a generalized Cauchy-Schwartz that we prove in Appendix \ref{appdx:sec:aux-claims} (Claim \ref{cl:general-CS}).
\end{proof}

\begin{lemma}\label{lem-main:exploratory_policies}
 With high probability, 
\begin{align}\label{eq:lem:exploratory_policies}
    \sum_{j=1}^{2d}\| \Delta_{t_r}^T\|_{\Sigma(M_{r,j})}^2\leq  2^{-2r}\cdot \frac{\gamma^2}{ 12^2\cdot 18 \cdot d\kappa^4}\ .
\end{align}
\end{lemma}
\begin{proof}
We will first need another lemma.
\begin{lemma}\label{lem:explor-policies-ls}
With high probability, for all epochs $r$, we have $\sum_{j=1}^{2d} \Sigma(M_{r,j})\preccurlyeq O(1/T_r)\cdot V_{t_r}$.
\end{lemma}
Combining with Lemma \ref{lem:LS}, we get that with high probability \begin{align}
    \sum_{j=1}^{2d} \|\Delta_{t_r}^T\|_{\Sigma(M_{r,j})}^2\leq O(1)\cdot \frac{\|\Delta_{t_r}^T\|_{V_{t_r}}^2 }{T_r}\leq \tO(1) \cdot \frac{(d_x+d_u)^2}{T_r}\ .
\end{align}
Substituting our choice of $T_r$ gives Lemma \ref{lem-main:exploratory_policies}.
 We now prove Lemma \ref{lem:explor-policies-ls}.
\begin{proof}
We will use the definitions introduced in Appendix \ref{appdx:prf:difference}, and we will show that with high probability, for all $r,j,h$, we have \footnote{With $|I_{r,j,h}|$ we denote the cardinality of this set.}
\begin{align}\label{eq:covariances}
   |I_{r,j,h}|\cdot \Sigma(M_{r,j}) \preccurlyeq O(1)\cdot \sum_{t\in I_{r,j,h}}z_tz_t^T +\frac{\lambda}{2dH'} \cdot I.
\end{align}
Once we have \ref{eq:covariances}, we can finish the proof of Lemma \ref{lem:explor-policies-ls}. Indeed, summing over all $j,h$ gives
\begin{align}
  \sum_{j=1}^{2d} (|I_{r,j}|-H') \cdot \Sigma(M_{r,j}) &\preccurlyeq O(1)\cdot \sum_{j=1}^{2d} \sum_{h=0}^{H'-1}\sum_{t\in I_{r,j,h}}z_tz_t^T +\lambda \cdot I  \notag \\
  & \preccurlyeq O(1)\cdot \sum_{t=1}^{t_r-1}z_sz_s^T +\lambda \cdot I \preccurlyeq O(1)\cdot V_{t_r},
\end{align}
where we used that $\sum_{h=0}^{H'-1}|I_{r,j,h}|=|I_{r,j}|-H'$. Now, since $|I_{r,j}|-H'=T_r-H'\geq T_r/2$, we get $\sum_{j=1}^{2d}  \Sigma(M_{r,j})\preccurlyeq O(1/T_r)\cdot V_{t_r}$.
\\
\\
We now prove \ref{eq:covariances}. We will use the auxiliary sequence $(\oz_t)_t$. From Appendix \ref{appdx:sec:aux-claims} (Claim \ref{appdx:cl:aux-sequence}), we have that with high probability, for all $t$,
\begin{align}\label{eq:approx-aux-sequence}
        \|\overline{z}_t-z_t\|\leq \tO(\kappa^2\beta\gamma^{-1}G)\cdot \sum_{i=1}^{2H+1}\|\whw_{t-i}-w_{t-i}\|+1/T.
\end{align}
Furthermore, we will show the following claim.
\begin{claim}\label{cl:concentration}
With high probability, for all $r,j,h$ we have
\begin{align}\label{eq:concentration}
   |I_{r,j,h}|\cdot\Sigma(M_{r,j})\preccurlyeq O(1)\cdot \sum_{t\in I_{r,j,h}}\oz_t\oz_t^T. 
\end{align}
\end{claim}
We show how to prove \ref{eq:covariances}, using the claim above and inequality \ref{eq:approx-aux-sequence}. After, we prove Claim \ref{cl:concentration}. Let $e_t=z_t-\oz_t$ and $p=\kappa^2\beta\gamma^{-1}G$. We condition on the event that the bounds \ref{eq:approx-aux-sequence}, \ref{eq:concentration} and that of Lemma \ref{lem-main:noise-approx} hold. We have
\begin{align}
 \|e_t\|^2 &\leq 2/T^2+O(p^2)(2H+1)\sum_{i=1}^{2H+1}\|w_{t-i}-\whw_{t-i}\|^2,
\end{align}
where we used \ref{eq:approx-aux-sequence} and Cauchy-Schwarz.
We now fix a triple $(r,j,h)$. Summing over all $t\in I_{r,j,h}$, 
\begin{align}
\sum_{t\in I_{r,j,h}} \|e_t\|^2 &\leq  \frac{2}{T}+O(p^2H) \cdot\sum_{t\in I_{r,j,h}} \sum_{i=1}^{2H+1}\|w_{t-i}-\whw_{t-i}\|^2 \notag \\
&\leq \frac{2}{T}+O(p^2  H)\cdot\sum_{t=1}^{T}\|w_{t}-\whw_{t}\|^2,
\end{align}
where we used the definition of $I_{r,j,h}$. Lemma \ref{lem-main:noise-approx} implies that with high probability, for all $r,j,h$,
\begin{align}\label{eq:total_noise_covariance}
\sum_{t\in I_{r,j,h}} \|e_t\|^2 &\leq  \widetilde{O}(\kappa^4\beta^2\gamma^{-3}G^2 )\cdot(d_x+d_u)^3.
\end{align}
Moreover, $\oz_t\oz_t^T=(z_t-e_t)(z_t-e_t)^T \preccurlyeq 2z_tz_t^T+2e_te_t^T$, so
\begin{align}\label{eq:total_noise_covariance-2}
    \sum_{t\in I_{r,j,h}}\oz_t\oz_t^T &\preccurlyeq 2\sum_{t\in I_{r,j,h}}z_t z_t^T+2\sum_{t\in I_{r,j,h}}e_te_t^T \preccurlyeq 2\sum_{t\in I_{r,j,h}}z_t z_t^T+2\left(\sum_{t\in I_{r,j,h}}\|e_t\|^2\right)\cdot I,
\end{align}
where we used that $\|\sum_{t\in I_{r,j,h}}e_te_t^T\|\leq \sum_{t\in I_{r,j,h}}\|e_te_t^T\|=\sum_{t\in I_{r,j,h}}\|e_t\|^2$. Now, Claim \ref{cl:concentration} and inequalities \ref{eq:total_noise_covariance} and \ref{eq:total_noise_covariance-2} give that with high probability, for all $r,j,h$,
\begin{align}
    &|I_{r,j,h}|\cdot\Sigma(M_{r,j})\preccurlyeq O(1)\cdot \left(2\left(\sum_{t\in I_{r,j,h}}\|e_t\|^2\right)\cdot I+2\sum_{t\in I_{r,j,h}} z_t z_t^T  \right) \notag \\
    & \preccurlyeq  \tO(\kappa^4\beta^2\gamma^{-3}G^2)(d_x+d_u)^3 \cdot I + O(1)\cdot\sum_{t\in I_{r,j,h}} z_t z_t^T  \notag \\
    & \preccurlyeq \frac{\lambda}{2dH'}\cdot I +O(1)\cdot \sum_{t\in I_{r,j,h}} z_t z_t^T,
\end{align}
where we used that $\lambda=\widetilde{\Theta}(\kappa^{4}\beta^2\gamma^{-5}G^2)\cdot d_xd_u(d_x+d_u)^3$. It remains to prove Claim \ref{cl:concentration}. We use the following lemma, which is Theorem 1.1 of \cite{hsu2012random}.

\begin{lemma}[\cite{hsu2012random}]\label{lem:relative-cov}
There exist positive constants $c_1,c_2$, such as the following hold. Let $\Sigma \in \mathbb{R}^{m\times m}$ positive semidefinite and $z_1,\dots,z_n$ independent random vectors, distributed as $N(0,\Sigma)$. Let $\widehat{\Sigma}=1/n\cdot \sum_{i=1}^n z_iz_i^T$. For all $\delta>0$, there exists $c_3=polylog(m,1/\delta)$, such that if $n\geq c_3\cdot m$, then with probability at least $1-\delta$, $range(\widehat{\Sigma})=range(\Sigma)$ and $\|\Sigma^{1/2} \widehat{\Sigma}^{\dagger}\Sigma^{1/2}\| \leq c_2$. \footnote{The symbol $\dagger$ denotes the pseudoinverse.}
\end{lemma}
We can immediately get the following corollary.
\begin{corollary}\label{cor:relative-cov}
For the setting of Lemma \ref{lem:relative-cov}, with probability at least $1-\delta$, we have $\Sigma \preccurlyeq O(1) \cdot \widehat{\Sigma}$.
\end{corollary}
\begin{proof}
We have $\bar{\sigma}_{min} ((\Sigma^\dagger)^{1/2} \widehat{\Sigma}(\Sigma^\dagger)^{1/2})\geq 1/\Omega(1)$, where $\bar{\sigma}_{min}$ denotes the minimum nonzero singular value. Let $P$ be the projection matrix on $range(\Sigma)$. We have
\begin{align}
    \widehat{\Sigma}=P\widehat{\Sigma}P=\Sigma^{1/2}  (\Sigma^\dagger)^{1/2}\widehat{\Sigma}(\Sigma^\dagger)^{1/2} \Sigma^{1/2}.
\end{align}
So, for all $x\in \mathbb{R}^m$, we have 
\begin{align}
     x^T \widehat{\Sigma} x=x^T \Sigma^{1/2}  (\Sigma^\dagger)^{1/2}\widehat{\Sigma}(\Sigma^\dagger)^{1/2} \Sigma^{1/2}x\geq \bar{\sigma}_{min} ((\Sigma^\dagger)^{1/2} \widehat{\Sigma}(\Sigma^\dagger)^{1/2}) x^T\Sigma x.
\end{align}
\end{proof}
Now, we apply Corollary \ref{cor:relative-cov} to show Claim \ref{cl:concentration}. We fix a triple $(r,j,h)$. Notice that $M_{r,j}$ is a random variable that depends only on the disturbances that took place up to epoch $r-1$. On the other hand, the random vectors $(\oz_t)_{t\in I_{r,j,h}}$ are independent of each other and independent of all the disturbances that took place up to epoch $r-1$, which follows from the definitions of $I_{r,j,h}$ and of the sequence $(\oz_t)_t$. Thus, we have $\mathbb{E}_w\left[\oz_t\oz_t^T\ |\ M_{r,j}\right]=\Sigma(M_{r,j})$, where $w$ denotes the sequence $(w_t)_t$. So, after conditioning on $M_{r,j}$, we can apply Corollary \ref{cor:relative-cov} with $\Sigma = \Sigma(M_{r,j})$ and the set of vectors being $(\oz_t)_{t\in I_{r,j,h}}$. Since $|I_{r,j,h}|=T_r/H'$ and $T_r$ is chosen to be large enough, we get that with high probability, $ |I_{r,j,h}|\cdot\Sigma(M_{r,j})\preccurlyeq O(1)\cdot \sum_{t\in I_{r,j,h}}\oz_t\oz_t^T$. This was for a fixed $r,j,h$, so union bound concludes the proof.
\end{proof} 
\end{proof}

\subsection{Proof of Lemma \ref{lem-main:stat-regret-bound}}\label{appdx:prf:lem-main:stat-regret-bound}
 To bound $R_T^{avg}$, we bound the suboptimality gap, i.e. $R^{avg}(M):=\mC(M\ |
  \ A_*,B_*)-\mC(M_*\ |\ A_*,B_*)$, for all policies in $\mM_{r+1}$.
 \begin{lemma}\label{lem-main:regr_each_step}
 With high probability, for all epochs $r$, we have
\begin{itemize}
\item $M_*\in \mM_r$, and
\item for all $M\in \mM_{r+1}$, $R^{avg}(M)\leq 5\cdot2^{-r}$. 
\end{itemize}
\end{lemma}
\begin{proof}
We condition on the event that for all $r,M\in \mM_r$,
\begin{align}\label{eq:loss_bound}
    \left|\mC(M\ |\ \whA_{t_r} , \whB_{t_r})-\mC(M\ |\ A_*,B_*)\right|\leq 2^{-r}.
\end{align}
 For the first bullet of the lemma, suppose that for some $r$, $M_*\in \mM_r$ and $M_*\notin \mM_{r+1}$. Thus, there exists $M\in \mM_r$, such that $\mC(M\ |\
 \whA_{t_r} , \whB_{t_r})<\mC(M_*\ |\ \whA_{t_r} , \whB_{t_r})-3\epsilon_r$. Then, inequality \ref{eq:loss_bound} implies that $\mC(M\ |\ A_* , B_*)-\epsilon_r<\mC(M_*\ |\ A_* , B_*)+\epsilon_r-3
\epsilon_r$,
which contradicts the optimality of $M_*$. 
\par For the second bullet, if $M\in \mM_{r+1}$, then $\mC(M\ |\ \whA_{t_r} , \whB_{t_r})-\mC(M_*\ |\ \whA_{t_r} , \whB_{t_r})\leq 3\epsilon_r$, since we showed that $M_*\in \mM_r$. By applying inequality \ref{eq:loss_bound}, we get $\mC(M\ |\ A_*, B_*)-\mC(M_*\ |\ A_* , B_*)\leq 5\epsilon_r$.
\end{proof}
Now, we will finish the proof of Lemma \ref{lem-main:stat-regret-bound}.

 $R_T^{avg}=\sum_{t=1}^TR^{avg}(M_t)=\sum_{r=1}^{q}\sum_{j=0}^d T_{r}\cdot R^{avg}(M_{r,j})$, where $q$ is the total number of epochs. Since $M_{r,j}\in \mM_r$, \ref{lem-main:regr_each_step} implies that with high probability, for all $r\geq 2$ and for all $j$, we have $R^{avg}(M_{r,j})\leq 5\cdot 2^{-(r-1)}$. We will now bound $\sum_{r=1}^q2^{r}$. Observe that $T_r=   D\cdot 2^{2r}$, where $D=\widetilde{\Theta}(\kappa^4\gamma^{-3}) \cdot d_xd_u(d_x+d_u)^2$, and $T\gtrsim\sum_{r=1}^qd\cdot T_r=D\cdot d\sum_{r=1}^q2^{2r}$. Thus, $\sqrt{\frac{T}{Dd}}\gtrsim \sqrt{\sum_{r=1}^q2^{2r}}\geq q^{-1/2}\sum_{r=1}^q2^{r}$, by Cauchy-Schwarz. Using that $q\leq O\left(\log{T}\right)$, we get $\sum_{r=1}^q2^{r}\leq \widetilde{O}(1)\cdot \sqrt{\frac{T}{Dd}}$. Summarizing, by excluding the first epoch, we have
\begin{align}
  R_T^{avg}-\sum_{j=1}^{2d} T_{1}\cdot R^{avg}(M_{1,j})\leq  \sum_{r=2}^{q}\sum_{j=0}^d D \cdot2^{2r}\cdot 5\cdot2^{-(r-1)} \lesssim D\cdot d\cdot \sum_{r=1}^q2^{r} \leq \widetilde{O}(1)\cdot \sqrt{D\cdot d \cdot T}.
\end{align}
In Appendix \ref{appdx:sec:aux-claims} (Claim \ref{cl:loss_range}), we show that for all $M\in \mM$, $R^{avg}(M)\leq \widetilde{O}(\kappa^4\beta^2\gamma^{-1}G) \cdot \sqrt{d_x} $. Thus, 
\begin{align}
  R_T^{avg}\leq \widetilde{O}(1)\cdot \sqrt{D\cdot d \cdot T}+ 2d\cdot 2^{2}\cdot D \cdot \widetilde{O}(\kappa^4\beta^2\gamma^{-1}G)\cdot \sqrt{d_x}.
\end{align}
By substituting $D=\widetilde{\Theta}(\kappa^4\gamma^{-3})\cdot d_xd_u(d_x+d_u)^2$ and $d=\widetilde{O}(\gamma^{-1})\cdot d_xd_u$, we get that $R_T^{avg}\leq \widetilde{O}\left(\kappa^2\gamma^{-2}\right)\cdot d_xd_u(d_x+d_u)\sqrt{T}+ \tO\left(\kappa^8\beta^2 \gamma^{-5}G \right)\cdot (d_x+d_u)^{6.5}$.
\qed
 


\section{Bandit feedback: Proof of Theorem \ref{thm-main:bandit}}\label{appdx:bandit}
We will prove the following theorem.
\begin{theorem}\label{appdx:thm:initial-condition}
  There exist $C_1,C_2,C_3,C_4,C_5=  \poly\left(d_x,d_u,\kappa,\beta,\gamma^{-1},G,\log{T}\right)$, such that after initializing the SBCO algorithm with $d=d_x\cdot d_u\cdot H$, $D=C_1$, $L=C_2$, $\sigma^2=C_3$ and $n=T/(2H+2)$, the following holds. If $T\geq C_4$, the intial state $\|x_1\|\leq \tO(\kappa^2\beta\gamma^{-1/2})\cdot \sqrt{d_x}$, and the initial estimation error bound $\|(A_0 \ B_0)-(A_* \ B_*)\|_F\leq \epsilon$ satisfies $ \epsilon^2 \leq \left(C_6\cdot d_xd_u(d_x+d_u)\right)^{-1}$, where $C_6=\kappa^{4}\beta^{2}\gamma^{-5}G^2$, then with high probability, Algorithm \ref{alg:bandit-feedback} satisfies $R_T\leq C_5 \cdot \sqrt{T}.$
\end{theorem}
Given the above theorem, Theorem \ref{thm-main:bandit} follows from the analysis of warmup exploration given in Appendix \ref{appdx_subs:warm-up} (specifically Lemma \ref{lem:warmup}). Theorem \ref{appdx:thm:initial-condition} follows from the following two lemmas (similarly to the case of known cost function). 
\begin{lemma}\label{lem-bandit:surrogate-state-bound}
With high probability, $R_T-R_T^{avg}\leq \poly\left(d_x,d_u,\kappa,\beta,\gamma^{-1},G,\log{T}\right)\cdot \sqrt{T}$. 
\end{lemma}
\begin{lemma}\label{lem-bandit:stat-regret-bound}
With high probability, $R_T^{avg}\leq \poly\left(d_x,d_u,\kappa,\beta,\gamma^{-1},G,\log{T}\right)\cdot \sqrt{T}$.
\end{lemma}
To prove these lemmas, we will first need a bound for $\sum_{t=1}^T\|\whw_t-w_t\|^2$.
\begin{lemma}\label{lem:bandit:noise-approx}
With high probability, Algorithm \ref{alg:bandit-feedback} satisfies $\sum_{t=1}^T \|\widehat{w}_t-w_t\|^2\leq \widetilde{O}(1)\cdot (d_x+d_u)^3.$
\end{lemma}
The proof of the lemma is exactly the same with the proof of Lemma \ref{lem-main:noise-approx}, so we do not repeat it here. Second, we require a generalization of the SBCO setup (Appendix \ref{sec:robust}). After this, we prove Lemma \ref{lem-bandit:surrogate-state-bound} in Appendix \ref{appdx:prf:lem-bandit:surrogate-state-bound} and Lemma \ref{lem-bandit:stat-regret-bound} in Appendix \ref{appdx:prf:lem-bandit:stat-regret-bound}.

\subsection{SBCO: robustness to small adversarial perturbations and low number of swtiches}\label{sec:robust}
We consider a small generalization of the SBCO setup, where the learner observes the function values under the sum of a stochastic and a small (on average) adversarial corruption. We will show that we can properly set the hyperparameters of the SBCO algorithm from \cite{agarwal2011stochastic}, to get $\sqrt{n}$ regret efficiently ($n$ is the time horizon), in this more general setting. We will also note some useful properties of that algorithm and we will finally give some preliminaries related to its application in Algorithm \ref{alg:bandit-feedback}.
 \subsubsection*{Setting}
Let $\mX$ be a convex subset of $\mathbb{R}^d$, for which we have a separation oracle and has diameter bounded by $D$. Let $f : \mX \rightarrow \mathbb{R}$  be an $L$-Lipschitz
convex function on $\mX$. We have noisy black-box access to $f$. Specifically, we are allowed to do $n$ queries: at time $t$ we query $x_t$ and the response is
\begin{equation}
    y_t = f(x_t) + \zeta_t+\xi_t
\end{equation}
where $\zeta_t$ conditioned on $(\zeta_1,\dots,\zeta_{t-1})$ is $\sigma_\zeta^2$-subgaussian with mean $0$ \footnote{In \cite{agarwal2011stochastic} they consider $\zeta_t$ independent but the analysis easily generalizes to the martingale condition that we use.}. The sequence $\xi_1,\dots,\xi_n$ can be completely adversarial and can even depend on $\{\zeta_t\}_{t\in[n]}$. However, the magnitude of this adversarial noise satisfies the following constraint: with probability at least $1-1/n^c$,
\begin{equation}
    \sum_{t=1}^n\xi_t^2\leq \sigma_{\xi}^2,
\end{equation}
for some parameters $c$ \footnote{This of $c$ as a large constant.}, $\sigma_\xi\geq 0$. The algorithm incurs a cost $f(x_t)$ for the query $x_t$. The goal is to minimize regret:
\begin{equation}
    \sum_{t=1}^n\left(f(x_t)-f(x_*)\right),
\end{equation}

where $x_*$ is a minimizer of $f$ over $\mX$. Clearly, the standard SBCO setting \cite{agarwal2011stochastic} is recovered when $\xi_t=0$ for all $t$. The algorithm in \cite{agarwal2011stochastic} uses a hyperparameter $\sigma$, which is set to be $\sigma_\zeta$. We will show that for this more general setting that we described, we can get the same regret guarantee (up to a factor depending on $\sigma_\xi$), by setting $\sigma:=\sqrt{c+1}\cdot\max(\sigma_\zeta,\sigma_\xi)$ and running the same algorithm.
\subsubsection*{Regret bound}
\begin{theorem}\label{thm:sbco}
 With probability at least $1-O(n^{-c})$, the algorithm in \cite{agarwal2011stochastic} (page 11) initialized with hyperparameter $\sigma=\sqrt{c+1}\cdot\max(\sigma_\zeta,\sigma_\xi)$ has regret 
\begin{equation}\label{eq:regret_stoch_bco}
     \sum_{t=1}^T\left(f(x_t)-f(x_*)\right)\leq \poly\left(\sigma,d,L,\log n, \log D\right)\cdot \sqrt{n}.
\end{equation}
\end{theorem}
\begin{proof}
 Every time this algorithm queries a new point $x$, it queries it multiple times and takes the average of the responses to reduce the variance. More specifically, the algorithm maintains a parameter $\gamma$ which is the desired estimation accuracy. If at time $t$, the point to be queried is new (different than the one at time $t-1$), then it queries it $s=4\cdot \frac{\sigma^2}{\gamma^2} \log n$ times\footnote{In \cite{agarwal2011stochastic}, there is a typo, because they write $s=2\cdot \frac{\sigma}{\gamma^2} \log n$. However, they fixed it in the journal version \cite{agarwal2013stochastic}, where the formula for $s$ is the one we give here.}) and receives 
$y_{t},\dots,y_{t+s-1}$. Then, the algorithm computes the average $avg_t=1/s\cdot\sum_{i=0}^{s-1}y_{t+i}$. In \cite{agarwal2011stochastic}, the proof of the regret bound (which is the same as the RHS of \ref{eq:regret_stoch_bco}) uses the fact that the noise is stochastic only in order to argue that with probability at least $1-\delta$, the error $|avg_t-f(x_t)|\leq \gamma$, for all $t$. Once they have this, their analysis implies that the regret bound holds with probability at least $1-\delta$. The proof of our theorem is essentially that this condition also holds in our setting (for $\delta=1-O(n^{-c})$), if we set $\sigma=\sqrt{c+1}\cdot \max(\sigma_\zeta,\sigma_\xi)$. Indeed, we have
\begin{align}
    avg_t=f(x_t)+\frac{\sum_{i=0}^{s-1}\zeta_{t+i}}{s}+\frac{\sum_{i=0}^{s-1}\xi_{t+i}}{s}.
\end{align}
\begin{itemize}
    \item Stochastic component: $s=4\cdot  \frac{\sigma^2}{\gamma^2} \log n\geq (c+1)\cdot \frac{\sigma^2_\zeta}{(\gamma/2)^2} \log n$, so from Azuma's inequality: $\left|\frac{\sum_{i=0}^{s-1}\zeta_{t+i-1}}{s}\right|\leq \gamma/2$, with probability at least $1-O(n^{-(c+1)})$. A union bound implies that the bound holds for all $t$, with probability at least $1-O(n^{-c})$.
    \item Adversarial component: by applying Cauchy-Schwarz, we get that with probability at least $1-O(n^{-c})$, for all $t$,
    \begin{align}
   \left| \frac{\sum_{i=0}^{s-1}\xi_{t+i}}{s}\right|\leq \frac{\sqrt{s}\sqrt{\sum_{i=0}^{s-1}\xi_{t+i}^2}}{s}\leq \frac{\sigma_\xi}{\sqrt{s}}\leq \gamma/2.
\end{align}
\end{itemize}
\end{proof}

Other than the regret guarantee, we will also need some other properties of the SBCO algorithm. To present these, we need a high level description of this algorithm, which we now provide. 
\subsubsection*{High level description of the SBCO algorithm}
Let $H_t=(x_i,y_i)_{i=1}^t$\footnote{We define $H_0=\emptyset$.}, i.e., the history up to time $t$. There exists a function $g$ that is polynomial-time computable and takes as input $H_t$ (for any $t$) and outputs a pair $(x,s)$, which indicates that the algorithm will query $x$ for the timesteps $t+1,t+2,\dots,t+s$. More specifically, given this function $g$, the SBCO algorithm has the following form.
{
\setlength{\interspacetitleruled}{-.4pt}%
\begin{algorithm}
Set $t=1$.\\
Set $r=1$.\\
\While{$t\leq n$}{
    Set $j_r=t$ (switching time).\\
    Set $(x,s)=g(H_{t-1})$. \\
    Query $x$ for the timesteps $t,t+1,\dots,t+s-1$.\\
    Set $t=t+s$.\\
    Set $r=r+1$.
}
\end{algorithm}
}

The way the function $g$ is constructed makes sure that the above algorithm queries exactly $n$ points. We now state two facts about this algorithm, the first follows from the above description and the second from inspecting the full algorithm (page 11 of \cite{agarwal2011stochastic}).
\begin{fact}\label{fact:measurable}
If $t\in[j_r,j_{r+1})$, then $x_t$ (point queried at time $t$) is $\sigma(H_{j_r-1})$-measurable.
\end{fact}
\begin{fact}\label{fact:switches}
At the end of the algorithm, the index $r\leq poly(\log{n},d,\sigma,\log{D},L)$.
\end{fact}
Note that Fact \ref{fact:switches} says that the number of point-switches is only logarithmic in $n$.

\subsubsection*{SBCO algorithm in Algorithm \ref{alg:bandit-feedback}: preliminaries}
Let $H'=2H+1$, $n=\lfloor T/H' \rfloor$ and $M(1),M(2),\dots,M(n)$ be the points/policies queried by the SBCO algorithm in Algorithm \ref{alg:bandit-feedback}. Observe that for all $t$, if $t=(j-1)H'+h$ for some $h\in \{1,\dots,H'\}$, then the executed policy at time $t$ is $M_t=M(j)$. Also, let $j_1\leq j_2\leq \dots\leq j_k$ be the switching timesteps of the SBCO algorithm (as in the high-level description of Section \ref{sec:robust}). We define $t_r=(j_r-1)H'+1$ and $t_{k+1}=T+1$. Observe that the executed policy $M_t$ remains constant for all $t\in [t_r,t_{r+1})$. Also, Fact \ref{fact:measurable} directly implies the following claim that we will use later.
\begin{claim}\label{cl:Mt:measurable}
If $t\in [t_r,t_{r+1})$, then $M_t$ is $\sigma((x_s,u_s)_{s=1}^{t_r-1})$-measurable.
\end{claim}

\subsection{Proof of Lemma \ref{lem-bandit:surrogate-state-bound}}\label{appdx:prf:lem-bandit:surrogate-state-bound}
The proof is similar to the proof of Lemma \ref{lem-main:difference}. We use $\mC(M)$ to denote $\mC(M\ |\ A_*,B_*)$.
\begin{align} 
  R_T-R_T^{avg}& = \sum_{t=1}^T \left(c(z_t)-\mC(M_t)\right)+T\min_{M\in\mM} \mC(M)-T\min_{M\in\mathcal{M}} J(M) \notag \\
  &\leq \sum_{t=1}^T \left(c(z_t)-\mC(M_t)\right) +1,
\end{align}
where we used Theorem \ref{thm:truncation}. We proceed with some definitions. For all $r\in \{1,\dots,k\}$, $h\in\{0,1,\dots,H'-1\}$, we define the intervals $I_{r}=[t_r,t_{r+1})$, $I_{r,h}=\{t\in I_r\ |\ t=t_r+H'\cdot j +h, j\geq 1\}$ and $I_{r}'= \{t_{r},t_{r}+1,\dots,t_{r}+H'-1\}=I_r \setminus (\cup_{h=0}^{H'-1}I_{r,h})$. We have
\begin{align}
    &\sum_{t=1}^T \left(c(z_t)-\mC(M_t)\right)=\sum_{r=1}^k\sum_{t\in I_r}\left(c(z_t)-\mC(M_t)\right) \notag \\
    &=\sum_{r=1}^k\sum_{h=0}^{H'-1}\sum_{t\in I_{r,h}}\left(c(z_t)-\mC(M_t)\right)+\sum_{r=1}^k\sum_{t\in I_{r}'}\left(c(z_t)-\mC(M_t)\right)
\end{align}
We call the first sum $S_1$ and the second $S_2$. In Appendix \ref{appdx:sec:aux-claims} (Claim \ref{appdx:cl:instant-regret}), we show that with high probability, $c(z_t)-\mC(M_t)\leq  \widetilde{O}(\kappa^4\beta\gamma^{-1}G)\cdot \sqrt{d_x}$, for all $t$. Combining with Fact \ref{fact:switches}, we get that with high probability, $S_2\leq k \cdot H'\cdot  \widetilde{O}(\kappa^4\beta\gamma^{-1}G)\cdot \sqrt{d_x}=\poly(d_x,d_u,\kappa,\beta,\gamma^{-1},G,\log{T})$. To bound $S_1$, we use the auxiliary sequence $(\oz_t)_t$, defined in  Appendix \ref{appdx:prf:difference}. 
\begin{align}\label{eq:surrogate-error}
   &S_1=\sum_{h=0}^{H'-1}\sum_{r=1}^k\sum_{t\in I_{r,h}}\left(c(\oz_t)-\mC(M_t)\right)+\sum_{r=1}^k\sum_{h=0}^{H'-1}\sum_{t\in I_{r,h}}\left(c(z_t)-c(\oz_t)\right) 
\end{align}
We call the first sum $S_3$ and the second $S_4$. We first bound $S_4$. In Appendix \ref{appdx:sec:aux-claims} (Claim \ref{appdx:cl:aux-sequence}), we show that with high probability, for all $t$,
\begin{align}
    \|\oz_t - z_t\|\leq \widetilde{O}(\kappa^2 \beta \gamma^{-1}G)\cdot \sum_{i=1}^{2H+1}\|\whw_{t-i}-w_{t-i}\|+1/T.
\end{align}
So, we have that with high probability, 
\begin{align}
    S_4\leq \sum_{t=1}^T \|z_t-\oz_t\|\leq 
    \widetilde{O}(\kappa^2\beta\gamma^{-2}G) \sum_{t=1}^T \|\whw_{t}-w_{t}\|+ 1
\end{align}
We now apply Lemma \ref{lem:bandit:noise-approx}, followed by Cauchy-Schwartz, to get that with high probability, $\sum_{t=1}^T \|\whw_{t}-w_{t}\|\leq \widetilde{O}(1)\cdot \sqrt{(d_x+d_u)^3T}$. Thus, we showed that $S_4\leq \poly(d_x,d_u,\kappa,\beta,\gamma^{-1},G,\log{T})\cdot \sqrt{T}$. The final step is to bound $S_3=\sum_{h=0}^{H'-1}S_{3,h}$, where $S_{3,h}=\sum_{r=1}^k\sum_{t\in I_{r,h}}\left(c(\oz_t)-\mC(M_t)\right)$. We will show that with high probability, for all $h$, $S_{3,h}\leq \poly(d_x,d_u,\kappa,\beta,\gamma^{-1},G,\log{T})\cdot \sqrt{T}$, which will conclude the proof. We will prove the following claim.

\begin{claim}\label{cl:azuma-conditions}
Let $\mF_t=\sigma(w_1,w_2,\dots,w_{t-H'-1})$. Let $r\in \{1,\dots,k\}$, $h\in \{0,\dots,H'-1\}$, $t\in I_{r,h}$. The following hold.
\begin{itemize}
    \item If $t'\leq t-H'$, then $c(\oz_{t'})-\mC(M_{t'})$ is $\mF_t$-measurable.
    \item $\mathbb{E}[c(\oz_{t})-\mC(M_{t})\ | \ \mF_t]=0$.
    \item Conditioned on $\mF_t$, $c(\oz_{t})-\mC(M_{t})$ is $\poly(\kappa,\beta,\gamma^{-1},G,\log{T})$-subgaussian.
\end{itemize}
\end{claim}
Given this claim, we can apply Azuma's inequality and a union bound to bound $S_{3,h}$, for all $h$. It remains to prove the claim.
\begin{proof}
First, we show that if $t'\leq t$, then $M_{t'}$ is $\mF_t$-measurable. Indeed, from Claim \ref{cl:Mt:measurable}, we get that $M_{t'}$ is $\sigma((x_s,u_s)_{s=1}^{t_r-1})$-measurable. Also, we have $\sigma((x_s,u_s)_{s=1}^{t_r-1})\subseteq \sigma(w_1,w_2,\dots,w_{t_{r'}-2})\subseteq \mF_t$, since $t\geq t_r+H'$.
\par Now, we show the first bullet. Let $t'\leq t-H'$. Then, from the argument above, $M_{t'}$ is $\mF_t$-measurable. Also, $\oz_{t'}$ is $\sigma(w_1,w_2,\dots,w_{t'-1})$-measurable and $\sigma(w_1,w_2,\dots,w_{t'-1})\subseteq \mF_t$, since $t\geq t'+H'$.
\par For the second bullet, notice that conditioned on $\mF_t$, the only source of randomness in $c(\oz_{t})-\mC(M_{t})$ are the $w_{t-H'},\dots, w_{t-1}$. Since $t\in I_{r,h}$, at time $t$ the policy $M_t$ has already been executed for the last $H'$ steps. Thus, $\mathbb{E}[c(\oz_{t})-\mC(M_{t})\ |\ \mF_t]=0$.

\par For the third bullet, it is easy to see that $\oz_t$ is $\poly(\kappa,\beta,\gamma^{-1},G)$-Lipschitz as a function of $(w_{t-H'},\dots, w_{t-1})$. This, combined with gaussian concentration \cite{vershynin2018high} completes the proof.
\end{proof}
\qed
\subsection{Proof of Lemma \ref{lem-bandit:stat-regret-bound}}\label{appdx:prf:lem-bandit:stat-regret-bound}
We first prove the following lemma.

\begin{lemma}\label{lem:bandit-regret-application}
Under the conditions of Theorem \ref{appdx:thm:initial-condition}, we have 
\begin{align}
    \sum_{j=1}^n (\mC(M(j))-\mC(M_*))\leq \poly(d_x,d_u,\kappa,\beta,\gamma^{-1},G,\log{T})\cdot \sqrt{n}.
\end{align}
\end{lemma}
Given this lemma, Lemma \ref{lem-bandit:stat-regret-bound} immediately follows, since $R_T=H'\cdot \sum_{j=1}^n (\mC(M(j))-\mC(M_*))$. We now give the proof of Lemma \ref{lem:bandit-regret-application}.
\begin{proof}
Clearly, there exist $C_1,C_2\leq  \poly(d_x,d_u,\kappa,\beta,\gamma^{-1},G,\log{T})$, such that $\mC(M)$ is $C_1$-Lipschitz and the diameter of $\mM$ is at most $C_2$. It suffices to show that when the SBCO algorithm queries $M(j)=M_t$, where $t=(j-1)H'+1$, the response $c(z_{t+H'-1})$ can be written as 
\begin{align}
    c(z_{t+H'-1})=\mC(M(j))+\zeta(j)+\xi(j),
\end{align}
where
\begin{itemize}
    \item conditioned on $\zeta(1),\dots,\zeta(j-1)$, the noise $\zeta(j)$ is $\poly(d_x,d_u,\kappa,\beta,\gamma^{-1},G,\log{T})$-subgaussian, and
    \item with high probability, $\sum_{j=1}^n \xi(j)^2\leq \poly(d_x,d_u,\kappa,\beta,\gamma^{-1},G,\log{T})$.
\end{itemize}
We will use the auxiliary sequence $(\oz_t)_t$ defined in  Appendix \ref{appdx:prf:difference}, to write 
\begin{align}
    c(z_{t+H'-1})=\mC(M_t)+ (c(\oz_{t+H'-1})-\mC(M_t)) + (c(z_{t+H'-1})-c(\oz_{t+H'-1})).
\end{align}
The second term is $\zeta(j)$ and the third is $\xi(j)$. The guarantee on $\zeta(j)$ follows from Claim \ref{cl:azuma-conditions}. For the guarantee on $\sum_{j=1}^n \xi(j)^2$, we have $
  \sum_{j=1}^n \xi(j)^2\leq \sum_{t=1}^T \|z_t-\oz_t\|^2$. By Claim \ref{appdx:cl:aux-sequence}, we have $\sum_{t=1}^T \|z_t-\oz_t\|^2\leq O(1)+\poly(\kappa,\beta,\gamma^{-1},G) \cdot \sum_{t=1}^T \|
 w_t-\whw_t \|^2$. Lemma \ref{lem:bandit:noise-approx} concludes the proof.

\end{proof}


\section{Warmup exploration }\label{appdx_subs:warm-up}

\begin{algorithm}[H]
Set $T_0=\lambda$, where $\lambda$ is defined in Algorithm \ref{alg:LS}.\\
\For{$t=1,2,\dots,T_0$}{
    Observe $x_t$. \\
    Play $u_t\sim N(0,I)$.
}
Set $V=\sum_{t=1}^{T_0}z_tz_t^T+(\kappa^2+\beta)^{-2}\cdot I$. \footnote{$z_t$ is defined as in Algorithm \ref{alg:LS}.} \\
Compute $(A_0\ B_0)=\sum_{t=1}^{T_0}x_{t+1}z_t^TV^{-1}$.
  \caption{Warmup exploration}
  \label{alg:warm-up}
\end{algorithm}

To get the initial estimates $A_0,B_0$ we conduct the warm  up exploration given in Algorithm \ref{alg:warm-up}. In the main text we use $x_1$ to denote the state after the warmup period (i.e., $x_{T_0+1}$). This "reset" of time is done for simplifying the presentation in the main text. From Theorem 20 and Appendix B.2 in \cite{cohen2019learning}, we automatically get the following lemma.

\begin{lemma}\label{lem:warmup-estimate}
Let $\Delta_0=(A_0\ B_0)-(A_*\ B_*)$. With high probability,
\begin{align}
    \|\Delta_0\|_F^2\leq \widetilde{O}(1)\cdot \frac{(d_x+d_u)^2}{T_0}= (C_1 d_xd_u(d_x+d_u))^{-1},
\end{align}
where $C_1=\kappa^{4}\beta^{2}\gamma^{-5}G^2$.
\end{lemma}
Now, we bound the regret incurred during the warmup exploration.
\begin{lemma}\label{lem:warmup}
 Let $C=\kappa^{8}\beta^3\gamma^{-6}G^3$. With high probability, the regret incurred during warmup exploration is at most $\tilde{O}(C)\cdot d_x^{3/2}d_u (d_x+d_u)^3$, and the state at the end of it, i.e, $x_{T_0+1}$ has norm $\|x_{T_0+1}\|\leq \tO(\kappa^2\beta\gamma^{-1/2})\cdot \sqrt{d_x}$.
\end{lemma}
\begin{proof}
We define the regret at step $t$ to be $c(z_t)-J(K_*)$, where $K_*\in \argmin_{K\in \mathcal{K}}J(K)$. 
We prove the following claim.
\begin{claim}
During warmup period, the regret at step $t$ is at most $\|z_t\| + \tO(\kappa^4\beta\gamma^{-1}G)\cdot \sqrt{d_x}$.
\end{claim}
\begin{proof}
 Let $(x_t^{M},u_t^{M})$ be the state-control pair under the execution of policy $M$, and let $\mathbb{E}_M$ denote the expectation with respect to this execution. For all $t$, we have 
\begin{align}
    c(z_t)- J(M_*)&= c(z_t)- \lim_{T \to \infty} \frac{1}{T}\cdot \mathbb{E}_{M_*} \sum_{t=1}^Tc(x_t^{M_*},u_t^{M_*}) \notag \\
    &=\lim_{T \to \infty} \frac{1}{T}\cdot\sum_{t=1}^T \mathbb{E}_{M_*} \left [ c(z_t) - c(x_t^{ M_*},u_t^{M_*}) \right ] \notag \\
    &\leq \lim_{T \to \infty} \frac{1}{T}\cdot\sum_{t=1}^T \mathbb{E}_{M_*} \left [ \|z_t\| + \|x_t^{ M_*}\| +\|u_t^{M_*}\| \right ] \notag \\
    &\leq \|z_t\|  + \tO(\kappa^4\beta\gamma^{-1}G)\cdot \sqrt{d_x}, \notag
\end{align}
where we used Claim \ref{cl:range}.
\end{proof}
Now, we use that $\|z_t\|\leq \|x_t\|+\|u_t\|$, and we bound $\|x_t\|$ and $\|u_t\|$.
\begin{itemize}
\item With high probability, $\|u_t\|\leq \tO(\sqrt{d_x})$, for all $t\in[T_0]$. Indeed, for $t\in[T_0]$ we have $u_t\sim N(0,I)$, so the bound on $\|u_t\|$ follows from standard concentration bounds for norms of gaussian vectors.
\item Now, we bound $\|x_t\|$. For all $t\in[T_0+1]$, $x_{t}\sim N\left(0,\Sigma_t\right)$, where
\begin{align}
    \Sigma_t=\sum_{i=0}^{t-2}A_*^i(I+B_*B_*^T)\left(A_*^T\right)^i
\end{align}
From Claim \ref{cl:powers_A}, we have $\left\|\left(A_*^T\right)^iA_*^i \right\|\leq  \left\|A_*^i \right\|^2\leq \kappa^4 (1-\gamma)^{2i}.$
Also, $\|I+B_*B_*^T\|\leq 1+\|B_*\|^2\leq 1+\beta^2 $. We conclude that 
\begin{align}
 \| \Sigma_t\| \leq (1+\beta^2)\kappa^4 \sum_{i=0}^{\infty}(1-\gamma)^{2i}\lesssim \beta^2\kappa^4\gamma^{-1}.
\end{align}
Now, $x_t\sim \Sigma_t^{1/2}z_t$, where $z_t\sim N(0,I)$. Thus, with high probability, for all $t\in [T_0 +1]$, $\|x_t\|\leq \tO(\beta\kappa^2\gamma^{-1/2})\cdot \sqrt{d_x}$.
\end{itemize}
 Since at each step we suffer regret at most $\tO(\kappa^4\beta\gamma^{-1}G)\cdot\sqrt{d_x}$ and the warmup period is the interval $\{1,2,\dots,T_0\}$ and $T_0=\widetilde{\Theta}(\kappa^{4}\beta^{2}\gamma^{-5}G^2) \cdot d_xd_u(d_x+d_u)^3$, we are done.
\end{proof}

\section{Auxiliary Claims}\label{appdx:sec:aux-claims}
\begin{claim}\label{cl:noise_bounded}
With high probability, for all $t$, 
\begin{align}\label{eq:noise_bound}
    \|w_t\|\leq \tO(\sqrt{d_x}).
\end{align}
\end{claim}
\begin{proof}
This follows from standard concentration of the norm of gaussian random vectors \cite{vershynin2018high}.
\end{proof}
\begin{claim}\label{cl:powers_A}
For all $i\in \mathbb{N}$, $\|A_*^i\|\leq \kappa^2(1-\gamma)^i$.
\end{claim}
\begin{proof}
Using Assumption \ref{assump:stability}, we have $\|A_*^i\|=\|Q\Lambda^i Q^{-1}\|\leq \kappa^2 (1-\gamma)^i$.
\end{proof}
\begin{claim}\label{cl:range}
There exists a $Z= \widetilde{O}(\kappa^4\beta\gamma^{-1}G)  \cdot \sqrt{d_x}$, such that the following hold. For any policy $M\in \mM$, we have $\mathbb{E}_{\eta}\|z(M|A_*,B_*,\eta)\|\leq Z$. Furthermore, suppose that $\|x_1\|\leq \tO(\kappa^2\beta\gamma^{-1/2})\cdot \sqrt{d_x}$, and that instead of executing our algorithms, we play $u_t=\sum_{i=1}^HM_t^{[i-1]}w_{t-i}$ for all $t$, where $(M_t)_t$ is an arbitrary policy sequence, such that $M_t\in \mM$, for all $t$. Then, with high probability, we have $\|z_t\|\leq Z$.
\end{claim}
\begin{proof}

 First, we fix a policy $M\in \mM$. For this proof, we write $u(M\ |\ \eta_{i:i+H-1})=\sum_{j=1}^HM^{{[j-1]}}\eta_{i+j-1}$, where $\eta_{i:i+H-1}$ denotes the sequence $\eta_{i},\eta_{i+1},\dots, \eta_{i+H-1}$. We have 
 $$\mathbb{E}_{\eta}\|z(M\ |\ A_*,B_*,\eta)\|\leq \mathbb{E}_{\eta}\|x(M\ |\ A_*,B_*,\eta)\|+\mathbb{E}_{\eta}\|u(M\ |\ \eta_{0:H-1})\|.$$
 Now, for all $i=\{0,1,\dots,H+1\}$, we have
\begin{align}\label{eq:bound-stat-controls}
 \Big( \mathbb{E}_\eta \left\|u(M\ |\ \eta_{i:i+H-1}) \right \| \Big)^2\leq  \mathbb{E}_\eta \left\|u(M\ |
  \ \eta_{i:i+H-1}) \right\|^2 &=tr\left(\sum_{j=1}^H\left(M^{[j-1]}\right)
^T\cdot M^{[j-1]}\right) \notag \\
 &\leq d_x \sum_{j=1}^H \left\|M^{[j-1]}\right\|^2\notag\\
 &\leq G^2 d_x\ . 
\end{align}
Thus, we bounded $\mathbb{E}_{\eta}\|u(M\ |\ \eta_{0:H-1})\|\leq G\sqrt{d_x}.$ Now, we write
\begin{align}
   x(M\ |\ A_*,B_*\eta)= \sum_{i=0}^H A_*^{i}\cdot \eta_{i} +\sum_{i=0}^H A_*^{i} B_* u( M\ | \ \eta_{i+1:i+H})
\end{align}
By triangle inequality,
\begin{align}
\mathbb{E}_\eta\|x(M\ |\ A_*,B_*\eta)\|& \leq \sum_{i=0}^H \|A_*^{i}\|\cdot \mathbb{E}_\eta\|\eta_{i}\| +\sum_{i=0}^H \|A_*^{i}\|\cdot  \|B_*\| \cdot \mathbb{E}_\eta \left\|u( M\ | \ \eta_{i+1:i+H})\right\|  \notag \\
&\leq \sqrt{d_x}\kappa^2 \sum_{i=0}^H(1-\gamma)^i+\kappa^2 \beta \sum_{i=0}^H (1-\gamma)^i\cdot G \cdot \sqrt{d_x},
 \end{align}
 where we used inequality \ref{eq:bound-stat-controls}. Thus, we get $\mathbb{E}_\eta\|x(M\ |\ A_*,B_*\eta)\|\leq \kappa^2\beta\gamma^{-1}G\cdot \sqrt{d_x}$.
\\
\par Now, we will bound $\|z_t\|$. First, we assumed that $\|x_1\|\leq \tO(\beta\kappa^2\gamma^{-1/2})\cdot \sqrt{d_x}$.
Also, the disturbance bound from Claim \ref{cl:noise_bounded} and the spectral bounds on $M_t^{[i]}$, imply that with high probability, for all t, we have $\|u_t\|\leq  G\sqrt{d_x}$. We now show that for large enough $Z= \widetilde{O}(\kappa^5\beta^2\gamma^{-2})  \cdot \sqrt{d_x}$, after conditioning on $\|x_1\|\leq \tO(\beta\kappa^2\gamma^{-1/2})\cdot \sqrt{d_x}$ and $\|w_t\|\leq \tO(\sqrt{d_x})$ and $\|u_t\|\leq  G\cdot \sqrt{d_x}$, for all $t$, we have $\|x_t\|\leq Z/2$, for all $t$.
\begin{align}
    \|x_{t+1}\|\leq \| A_*^{t}x_1\|+\sum_{i=0}^{t-1} \|A_*^{i}\|\cdot \|w_{t-i}\| +\sum_{i=0}^{t-1} \|A_*^{i}\|\cdot \|B_*\|\cdot\|u_{t-i}\| 
\end{align}
 Using the bounds on disturbances, controls and $\|x_1\|$, we get that $\|x_{t+1}\|$ is at most
\begin{align}\label{eq:state-bound-calculation}
     \tO(\beta\kappa^2\gamma^{-1/2})\cdot \sqrt{d_x}\cdot \kappa^2 + \sqrt{d_x}\cdot \kappa^2\cdot \sum_{i=0}^\infty (1-\gamma)^i+\kappa^2\beta  \cdot  \sum_{i=0}^\infty (1-\gamma)^i \cdot G\cdot \sqrt{d_x},
\end{align}
which is at most $Z/2$. Thus, $\|z_t\|\leq Z$.
\end{proof}
\begin{claim}\label{cl:range-aux}
For all $t\geq H+2$, $\mathbb{E}_{w}\left\|x_{t-H-1}^{(1)}\right\|,\mathbb{E}_{w}\left\|x_{t-H-1}^{(2)}\right\|\leq O(\kappa^2\beta\gamma^{-1}G)\cdot \sqrt{d_x}$\ .
\end{claim}
\begin{proof}
First, for all $t$, $\mathbb{E}_w\|u_t\|\leq \sum_{i=1}^{H}\|M^{[i-1]}\| \cdot \mathbb{E}_w\|w_{t-i}\|\leq G\cdot \sqrt{d_x}$\ . We have 
\begin{align}
    x_{t-H-1}^{(1)}=\sum_{i=1}^{t-H-3}\whA^i w_{t-H-2-i}+\sum_{i=1}^{t-H-3}\whA^i \whB u_{t-H-2-i}.
\end{align}
Also, in Claim \ref{cl:id-approx} we proved that $\whA$ is $(\kappa,\gamma/2)$-stronlgy stable. Also, we have $\|\whB-B_*\|\leq \gamma/(2\kappa^2)$, so $\|\whB\|\leq \beta+1$. Thus, we get 
\begin{align}
 \mathbb{E}_w \|  x_{t-H-1}^{(1)}\|& \leq \sum_{i=1}^{t-H-3}\left\|\whA^i \right\| \cdot \mathbb{E}_w \|w_{t-H-2-i}\|+\sum_{i=1}^{t-H-3} \left\|\whA^i \right\| \cdot \|\whB\|\cdot \mathbb{E}_w\|u_{t-H-2-i}\| \notag \\
 &\leq \sqrt{d_x}\kappa^2\sum_{i=0}^\infty (1-\gamma/2)^i+\kappa^2 (\beta+1)\sum_{i=0}^\infty (1-\gamma/2)^i \cdot G\cdot \sqrt{d_x} \notag \\
 &\lesssim \kappa^2\beta \gamma^{-1}G\cdot \sqrt{d_x}\ .
\end{align}
Since $A_*$ is $(\kappa,\gamma)$-stronlgy stable and $\|B_*\|\leq \beta$, the same calculation gives $\mathbb{E}_w \|  x_{t-H-1}^{(2)}\|\lesssim \kappa^2\beta \gamma^{-1}G\cdot \sqrt{d_x}$\ .
\end{proof}

\begin{claim}\label{cl:general-CS}
Let $\lambda_1,\dots,\lambda_n\in\mathbb{R}$ and $A_1,\dots,A_n$ matrices with compatible dimensions. Then,
\begin{align}
   \left( \sum_{j=1}^n\lambda_jA_j\right)\left( \sum_{j=1}^n\lambda_jA_j\right)^T\preccurlyeq \left(\sum_{j=1}^n\lambda_j^2\right) \left(\sum_{j=1}^n A_jA_j^T\right)
\end{align}
\end{claim}
\begin{proof}
Without loss of generality, it suffices to prove the result for the case where $A_j$ have each only one column. Then, for all vectors $x$, 
\begin{align}
   x^T\left( \sum_{j=1}^n\lambda_jA_j\right)\left( \sum_{j=1}^n\lambda_j A_j\right)^Tx =&\left( \sum_{j=1}^n\lambda_j   A_j^Tx\right)^2 \notag \\
   &\leq  \left(\sum_{j=1}^n\lambda_j^2\right) \left(\sum_{j=1}^n (A_j^Tx)^2\right) \notag  \\
   &= x^T  \left(\sum_{j=1}^n\lambda_j^2\right) \left(\sum_{j=1}^n A_jA_j^T\right)x.
\end{align}
\end{proof}

\begin{proof}
We use the quantity $z(M\ |\ A_*,B_*, \eta)$ that we define in equation \ref{eq:def:z(M)}.
   \begin{align}
       R^{avg}(M)&=\mC(M\ |\ A_*,B_*)-\mC(M_*\ |\ A_*,B_*)  \notag \\
       &\leq\mathbb{E}_{\eta} \|z(M\ |\ A_*,B_*,\eta)-z(M_*\ |\ A_*,B_*,\eta) \|  \notag \\
       &\leq \mathbb{E}_{\eta} \|z(M\ |\ A_*,B_*,\eta)\|+ \mathbb{E}_{\eta} \|z(M_*\ |\ A_*,B_*,\eta) \|  \notag \\
       &\leq  \widetilde{O}(\kappa^4\beta\gamma^{-1}G) \cdot \sqrt{d_x},
\end{align}
where we used the definition of $\mC$, that $c$ is 1-Lipschitz and Claim \ref{cl:range}.
\end{proof}


\begin{claim}\label{appdx:cl:aux-sequence}
For both Algorithms \ref{LJ-exploration} and \ref{alg:bandit-feedback}, we have that with high probability, for all $t$,
\begin{align}
    \|\oz_t - z_t\|\leq \widetilde{O}(\kappa^2 \beta \gamma^{-1}G)\cdot \sum_{i=1}^{2H+1}\|w_{t-i}-\widehat{w}_{t-i}\|+1/T.
\end{align}
\end{claim}
\begin{proof}
First, for all $t$, 
\begin{align}\label{eq:bound_on_u}
 \|\ou_t-u_t\|=\left\|\sum_{i=1}^H M_t^{[i-1]} (w_{t-i}-\widehat{w}_{t-i}) \right\|\leq G\sum_{i=1}^H \|w_{t-i}-\widehat{w}_{t-i}\|.   
\end{align}
Furthermore, $\|\ox_t-x_t\|\leq \sum_{i=0}^H \|A_*^i\|\|B_*\|\| \ou_{t-i-1}-u_{t-i-1}\|+\|A_*^{H+1}x_{t-H-1}\|.$ Claims \ref{cl:powers_A}  and \ref{cl:range} imply that with high probability we have $\|A_*^{H+1}x_{t-H-1}\|\leq \kappa^2(1-\gamma)^{H+1}\cdot\widetilde{O}(\kappa^4\beta\gamma^{-1}G)\sqrt{d_x}\leq 1/T$.
Using the bound \ref{eq:bound_on_u}, we get
\begin{align}
   \|\ox_t-x_t\|& \leq \sum_{i=0}^H \kappa^2\beta (1-\gamma)^i\sum_{j=1}^H G\|w_{t-i-j-1}-\widehat{w}_{t-i-j-1}\| +1/T \notag \\
   & \leq \kappa^2 \beta HG \sum_{i=1}^{2H+1}\|w_{t-i}-\widehat{w}_{t-i}\|+1/T.
\end{align}
Finally, 
\begin{align}
\|\oz_t-z_t\|&\leq \|\ox_t-x_t\|+ \|\ou_t-u_t\| \notag \\
& \leq  \kappa^2 \beta HG \sum_{i=1}^{2H+1}\|w_{t-i}-\widehat{w}_{t-i}\| +1/T+ G\sum_{i=1}^H  \|w_{t-i}-\widehat{w}_{t-i}\|  \notag \\
& \leq \widetilde{O}(\kappa^2 \beta \gamma^{-1}G)\cdot \sum_{i=1}^{2H+1}\|w_{t-i}-\widehat{w}_{t-i}\|+1/T.
\end{align}

\end{proof}
\begin{claim}\label{appdx:cl:instant-regret}
For both Algorithms \ref{LJ-exploration} and \ref{alg:bandit-feedback}, we have that with high probability, for all $t$,
\begin{align}
    c(z_t)-\mC(M_t)\leq \widetilde{O}(\kappa^4\beta\gamma^{-1}G) \cdot \sqrt{d_x}.
\end{align}
\end{claim}
\begin{proof}
 \begin{align}
     c(z_t)-\mC(M_t)&=c(z_t)-\mathbb{E}_{\eta} \left[
     c\left(z(M\ |\ A_*,B_*,\eta)\right) \right] \notag \\
     &= \mathbb{E}_{\eta}\left[c(z_t)- c\left(z(M\ |\ A_*,B_*,\eta)\right)\right] \notag \\
     & \leq  \mathbb{E}_{\eta} \left\|z_t- z(M\ |\ A_*,B_*,\eta)\right\|\notag \\
     &\leq \|z_t\|+\mathbb{E}_{\eta}\|z(M\ |\ A_*,B_*,\eta)\|.
 \end{align}
Lemma \ref{lem:range-z_t} and Claim \ref{cl:range} complete the proof. 
\end{proof}

\subsection{Proof of Theorem \ref{thm:truncation}}\label{apdx:prf:truncation}
We consider the dynamics $u_{t+1}^M=\sum_{i=0}^{H-1}M^{[i]}w_{t-i}$ and $x_{t+1}^M=A_*x_t^M+B_*u_t^M+w_t$, $x_1^M=0$.
Now, let $t\geq 2H+2$ and let $w$ denote the sequence of disturbances. Also, let $\eta_t(w)=(w_{t-1},w_{t-1},\dots,w_{t-2H-1})$, and observe that $u(M \ | \ \eta_t(w))=u_t^M$ and $x( M\ |\ A_*,B_*,\eta_t(w))=x_t^M-A_*^{H+1}x_{t-1-H}^M$. Thus,
\begin{align}
    \mathbb{E}_w\| x( M\ |\ A_*,B_*,\eta_t(w))-x_t(M)\| \leq \|A_*^{H+1}\|\cdot \mathbb{E}_w\|x_{t-1-H}^M\|.
\end{align}
Now, Claims \ref{cl:powers_A}, \ref{cl:range} and our choice for $H$ imply that $\|A_*^{H+1}\|\cdot \mathbb{E}_w\|x_{t-1-H}^M\|\leq 1/T$. Using this bound, we get
\begin{align}
   & |J(M)-\mC(M\ | \ A_*,B_*)|\notag \\
   & = \left|\lim_{T \to \infty} \frac{1}{T}\cdot \mathbb{E}_w\sum_{t=1}^Tc(x_t^M,u_t^M)-\mathbb{E}_wc \left(x\left(M\ |\ A_*,B_*,\eta_t(w)\right),u\left(M\ |\ \eta_t(w))\right)\right) \right| \notag \\
   &\leq  \lim_{T \to \infty} \frac{1}{T}\cdot \sum_{t=1}^T \mathbb{E}_w\left( \|x_t^M-x(M\ |\ A_*,B_*,\eta_t(w))\|+\|u_t^M-u(M\ |\ \eta_t(w))\|  \right) \notag \\
   &\leq  1/T.
\end{align}
\subsection{Proof of Theorem \ref{thm:lin-opt-oracle-affine}}\label{appdx:prf:thm:lin-opt-oracle-affine}
Let $v_0$ an arbitrary point in $S$ (we can get one with one call to the oracle). Let $S-v_0:=\{v-v_0\ |\ v\in S\}$. Clearly, $S-v_0$ is also compact and since $S$ is not contained in any proper affine subspace, $S-v_0$ is not contained in any proper linear subspace. Furthermore, the linear optimization oracle for $S$ is also a linear optimization oracle for $S-v_0$. Thus, Theorem \ref{thm:lin-opt-oracle-simple} implies that for any $C>1$ we can compute an affine C-barycentric spanner for $S-v_0$ in polynomial time, using $O(d^2 \log_C(d))$ calls to the oracle, which finishes the proof.

\section{DFCs}\label{appdx:disturbance-based}
We show that under the execution of the policy $M$, we have 
\begin{align}
    x_{t+1}=A_*^{H+1}x_{t-H}+\sum_{i=0}^{2H}\Psi_i(M\ |\ A_*,B_*)w_{t-i},
\end{align}
where 
\begin{align}
\Psi_i(M\ |\ A_*,B_*)=A_*^{i}\ind_{i\leq H}+\sum_{j=0}^HA_*^j  
    B_* M^{[i-j-1]}\ind_{i-j\in[1,H]}.    
\end{align}
This formula was derived in \cite{agarwal2019online} and we rederive it here for completeness.
\begin{align}
    x_{t+1}&=\sum_{i=0}^H A_*^i(w_{t-i}+B_*u_{t-i}) + A_*^{H+1}x_{t-H} \notag \\
    &= \sum_{i=0}^H A_*^i w_{t-i} +\sum_{i=0}^H A_*^i B_* \sum_{j=1}^HM^{[j-1]}w_{t-i-j}  + A_*^{H+1}x_{t-H} \notag \\
    &= \sum_{i=0}^H A_*^i w_{t-i} +\sum_{\ell=0}^{2H} \sum_{i=0}^H A_*^i B_*M^{[\ell-i-1]}w_{t-\ell} \ind_{\ell-i \in [1,H]}  + A_*^{H+1}x_{t-H}\notag \\
    &=A_*^{H+1}x_{t-H}+\sum_{i=0}^{2H}\Psi_i(M\ |\ A_*,B_*)w_{t-i}. \notag
\end{align}

\end{document}